\documentclass[journal]{IEEEtran}
\usepackage{amsfonts}           
\usepackage{bm}                 
\usepackage{dsfont}
\usepackage{amssymb}
\usepackage{graphicx}           
\usepackage[caption=false,font=footnotesize,labelfont=rm,textfont=rm]{subfig}
\usepackage{enumerate}
\usepackage{amsmath}
\usepackage{amsthm}             

\theoremstyle{plain} 
\newtheorem{theorem}{Theorem} 
\newtheorem{lemma}{Lemma}     
\theoremstyle{definition}
\newtheorem{remark}{Remark}   

\usepackage{algorithm}                  
\usepackage{algpseudocode}   
\usepackage[bookmarks=false]{hyperref}
\hypersetup{hidelinks}
\usepackage{xurl}
\usepackage{verbatim}                   
\usepackage{cite}

\begin{document}
	
\title{Timeliness-Oriented Scheduling and Resource Allocation in Multi-Region Collaborative Perception}
    
	\author{ 
     Mengmeng Zhu,~\IEEEmembership{Student Member,~IEEE}, Yuxuan Sun,~\IEEEmembership{Member,~IEEE}, Yukuan Jia, \\
     Wei Chen,~\IEEEmembership{Senior Member,~IEEE}, Bo Ai,~\IEEEmembership{Fellow,~IEEE}, and Sheng Zhou,~\IEEEmembership{Senior Member,~IEEE}
    
    \thanks{Mengmeng Zhu, Yuxuan Sun (Corresponding Author), Wei Chen and Bo Ai are with the School of Electronic and Information Engineering, Beijing Jiaotong University, Beijing 100044, China. (e-mail: \{mengmengzhu, yxsun, weich, boai\}@bjtu.edu.cn)
    
    Yukuan Jia and Sheng Zhou are with Beijing National Research Center for Information Science and Technology, Department of Electronic Engineering, Tsinghua University, Beijing 100084, China. (e-mail: jyk20@mails.tsinghua.edu.cn, sheng.zhou@tsinghua.edu.cn)
    }}



	\maketitle
    \begin{abstract}
    Collaborative perception (CP) is a critical technology in applications like autonomous driving and smart cities. It involves the sharing and fusion of information among sensors to overcome the limitations of individual perception, such as blind spots and range limitations. However, CP faces two primary challenges. First, due to the dynamic nature of the environment, the timeliness of the transmitted information is critical to perception performance. Second, with limited computational power at the sensors and constrained wireless bandwidth, the communication volume must be carefully designed to ensure feature representations are both effective and sufficient.
    This work studies the dynamic scheduling problem in a multi-region CP scenario, and presents a Timeliness-Aware Multi-region Prioritized (TAMP) scheduling algorithm to trade-off perception accuracy and communication resource usage. 
    Timeliness reflects the utility of information that decays as time elapses, which is manifested by the perception performance in CP tasks.
    We propose an empirical penalty function that maps the joint impact of Age of Information (AoI) and communication volume to perception performance.
    Aiming to minimize this timeliness-oriented penalty in the long-term, and recognizing that scheduling decisions have a cumulative effect on subsequent system states, we propose the TAMP scheduling algorithm. TAMP is a Lyapunov-based optimization policy that decomposes the long-term average objective into a per-slot prioritization problem, balancing the scheduling worth against resource cost.
    We validate our algorithm in both intersection and corridor scenarios with the real-world Roadside Cooperative perception (RCooper) dataset. Extensive simulations demonstrate that TAMP outperforms the best-performing baseline, achieving an Average Precision (AP) improvement of up to 27\% across various configurations.
    \end{abstract}	
    
    
    \section{Introduction}
     % 6G: Toward Comprehensive and Intelligent Perception
    % With the rapid evolution towards 6G networks, a new paradigm of comprehensive, intelligent perception is emerging \cite{jiang2021_6G}. 
    A new paradigm of comprehensive, intelligent perception is pivotal in applications such as autonomous driving and urban traffic monitoring \cite{jiang2021_6G}. Intelligent vehicles and roadside units perceive their surroundings using sensors like cameras and LiDARs. However, the effectiveness of an individual sensor is often compromised by a restricted field of view, a finite sensing range, and vulnerability to occlusions \cite{Xiao2023Occlusions}. To overcome these limitations, collaborative perception (CP) has emerged as a crucial solution \cite{han2023CP}. In CP, multiple sensors share sensing information via wireless channel aiming to expand the collective perceptual range and mitigate the performance degradation caused by occlusions and limited fields of view.
    
    % However, existing CP systems focus on collaboration within a single region that are insufficient for the expanded scope of applications, including traffic monitoring across numerous intersections and road segments \cite{Wang2025Grouping-Based}.
    However, existing CP approaches mainly focus on collaboration within a single region, treat each region as an independent entity. This design fails to account for the fact that regions may compete for shared communication and computational resources. Consequently, single region CP systems are insufficient for applications, such as traffic monitoring across numerous intersections and road segments \cite{Wang2025Grouping-Based}.    
    Thus we need to consider a multi-region CP system, which typically consists of two-levels. At the \emph{inter-region} level, a central Base Station (BS) orchestrates the operations and allocates resource across regions. At the \emph{intra-region} level, sensors within each region conduct CP to complement their individual fields of view and cover mutual blind spots. To facilitate CP, we adopt feature-level fusion \cite{wang2020v2vnet, xu2025v2x, xie2025s4, gao2025stamp}. This widely-used paradigm balances between two extremes: raw-level fusion, which transmits raw sensor data without information loss but incurs prohibitive bandwidth costs \cite{chen2019cooper}, and object-level fusion, which is bandwidth-efficient but may lose critical details \cite{arnold2020cooperative}.
    
     % Goal
    The primary goal of a multi-region CP system is to ensure the \emph{timeliness of information for all monitored regions}. 
    Timeliness refers to the value of sensing information, which diminishes rapidly due to dynamic environments \cite{zhou2024task}. Age of Information (AoI) is a widely adopted metric to quantify the freshness of information, measuring the time elapsed since the generation of the most recent received information \cite{kaul2012real, zhou2024task, sun2023optimizing}. 
    % For instance, if a cooperating sensor detects a passing vehicle but this information is processed after a significant delay, the vehicle may have already left the scenario. 
    For instance, if a cooperating sensor detects a passing vehicle, due to the mobility of vehicle, delayed sensing information results in inaccurate estimates of object positions.
    This information lag leads to a degradation in perception performance.
    
    % challenges
    To optimize multi-region CP performance, a critical problem arises: how can a BS efficiently manage massive data streams from multiple regions under constrained communication and computing resources? This problem can be decomposed into two key challenges.
    1) At the intra-region level, there is a fundamental \emph{trade-off between feature granularity and timeliness}. A larger communication volume provides richer information but increases transmission and computing latency, thereby degrading data freshness. The challenge is determining the optimal communication volume for each region to balance granularity-timeliness trade-off.
    2) At the inter-region level, the challenge is the \emph{complex region selection} problem. Due to limited bandwidth and server processing capability, only a subset of regions can be served simultaneously. 
    Some regions have been recently scheduled, while others have not been scheduled for a long time. As a result, they have different levels of scheduling urgency. This creates the need for a metric to quantify the real-time scheduling priority of each region. Moreover, the system is highly dynamic in terms of channel states and targets, necessitating an effective stochastic optimization policy that can make long-term decisions without requiring full knowledge of future system states.
    % Second, the scheduling problem is NP-hard due to its combinatorial complexity.
    % Therefore, a significant hurdle remains in designing an algorithm that jointly solves these coupled problems: first, determining how much data each region should send and second, deciding which regions to schedule in each slot.

    % existing work and research gap
    For \emph{intra-region} CP, existing research has first addressed to manage communication overhead. Feature-level fusion \cite{wang2020v2vnet, xu2025v2x, xie2025s4, gao2025stamp} has become the dominant paradigm. A focus within this paradigm is reducing data payloads by transmitting only the most salient information. To achieve this, researchers have developed various techniques. Task-adaptive codebooks \cite{hu2024pragmatic} have been proposed to ensure that only the information strictly necessary for the downstream task performed by a collaborator is transmitted. A channel-adaptive compression scheme \cite{zhou2024task} extracts the most valuable semantic information by adapting to real-time wireless communication constraints. The information bottleneck principle \cite{fang2024r-acp} has also been leveraged for an encoding method that adjusts video compression rates based on the task relevance of the content, thereby balancing accuracy and communication cost.
    
    Meanwhile, AoI has been widely adopted to quantify information freshness \cite{kaul2012real, zhou2024task, sun2023optimizing}. This linear metric measures the time elapsed since the generation of the most recently received information. Early works focused on minimizing AoI in real-time monitoring and control systems \cite{qin2022timeliness, kalor2022timely, qin2023timeliness}. However, the linear nature of AoI is often insufficient to capture the non-linear manner in which task performance degrades with time delay. 
    To address this, non-linear metrics have been proposed. For instance, Urgency of Information (UoI) \cite{zheng2020urgency} measures the non-linear, time-varying importance of status updates based on their context. Age of Usage Information (AoUI)  \cite{xie2023minimizing} jointly captures the freshness and usability of correlated data in IoT systems. For specific needs of multi-agent sensing, the Age of Perceived Targets (AoPT) \cite{fang2024r-acp} captures the collective data timeliness from multiple streaming views observing the same target.
    However, a critical gap remains in addressing their inherent trade-off. Existing frameworks lack a mechanism to characterize the timeliness requirements of each sensor considering its varying importance and data correlation, and thereby optimize individual communication volumes to enhance the overall timeliness and accuracy of the CP task.    
    
    For \emph{inter-region} scheduling, existing research can be viewed from two perspectives: designing priority metric and scheduling algorithm. First, the definition of scheduling priority of regions has evolved significantly. Age of Processed Information (AoPI) \cite{li2024towards} was introduced as a priority metric that integrates the recognition accuracy with transmission and computation efficiency, moving beyond pure timeliness. Customizable, task-specific penalty functions of AoI were formulated to define priority \cite{sun2023optimizing}, allowing the system to weigh freshness against communication and computation delays according to specific application needs. Additionally, priority metrics have been developed based on the direct ``perceptual gain" of a sensor to tasks, the importance and complementarity of sensor data in dynamic mobile environments, and implicit definitions derived from joint optimization problems aimed at minimizing execution delay \cite{hou2025enhancing, jia2025c-mass, xiao2022perception}.
    
    Second, to find an optimal algorithm based on a given priority metric, various scheduling models have been explored. Foundational works used model-based optimization, with Markov Decision Processes (MDPs). In this paradigm, researchers formulated scheduling problems as infinite horizon MDPs to minimize AoI, proving that the optimal policies often have a simple, threshold-based structure \cite{zhou2019minimum, kalor2022timely, tang2020minimizing}. However, to handle the complexity and dynamics of real-world environments, the field has adopted data-driven techniques like Deep Reinforcement Learning (DRL). DRL-based approaches can effectively tackle high-dimensional state and action spaces, ultimately learning near-optimal policies without needing a system model \cite{xie2021reinforcement, sun2021aoi, qin2023timeliness}.
    However, two gaps persist for multi-region CP scheduling. First, there is a lack of a practical, effective, timeliness-aware priority metric for CP. Because CP is transmission and computation-intensive, a significant latency exists between when a region is scheduled and when its data is processed and fused. Existing metrics often fail to account for this delay, making them poor predictors of future performance. Second, the current trend towards complex, data-driven solutions like DRL, with their high training overhead and computational demands, overlooks the need for more practical and lightweight scheduling policies that are better suited for real-time, resource-constrained environments. 
    
    % \subsection{Contributions}
    In the context of multi-region CP, this paper proposes a timeliness-oriented scheduling framework that dynamically selects regions and allocates resources to maximize global perception performance. Our main contributions are as follows:   
    \begin{itemize}
        \item We introduce a novel scheduling framework for multi-region CP. A new \emph{penalty function} tailored for the CP task is designed, modeling the non-linear degradation of perception performance by jointly considering the timeliness and communication volume of sensing information.
        \item We formulate the multi-region scheduling problem as a stochastic optimization problem.
        Using KKT conditions, we derive a scheduling priority metric capturing the persistent effects of decisions. We design \emph{{T}imeliness-{A}ware {M}ulti-region {P}rioritized ({TAMP})} scheduling algorithm, for region scheduling and resource allocation with resource constraints and system uncertainty.
        \item We validate our scheduling algorithm using the \emph{real-world roadside dataset} RCooper \cite{hao2024rcooper}. By establishing an empirical study with intersection and corridor scenarios, we fit the penalty function to inference data obtained from these scenarios and explore the practical performance of the proposed algorithm. This demonstrates the feasibility of our algorithm in realistic settings.
        \item Extensive simulations are conducted to evaluate the performance of the proposed algorithm for different scenario settings and rate distributions. Our results show that the proposed algorithm improves the Average Precision (AP) by up to 27\% compared to the baselines.
    \end{itemize}

    \section{System Model} \label{system model}
    \subsection{System Overview}
    As illustrated in Fig.~\ref{system1}, we consider a system composed of a base station (BS) co-located with an edge server, and a set of regions to be monitored, denoted by $\mathcal{A}$. Each region $ a \in \mathcal{A}$ is equipped with a set of sensors (e.g., cameras, LiDARs), which we represent by $\mathcal{N}_a$. The set of all sensors is denoted by $\mathcal{N}=\cup_{a \in \mathcal{A}}\mathcal{N}_a$. 
    To overcome the limitation of single view perception, sensors within each region perform CP. They then transmit their processed perceptual data to the BS for feature-level fusion and object detection. 
    % Upon completing a CP task, the freshness state of the corresponding region resets to ``fresh". This freshness decays over time, causing the state to transition to ``aging" and eventually ``stale". 
    Each time a CP task is completed, the BS obtains the latest information for that region, which then gradually becomes stale.
    Due to limited communication and computation resources, the BS schedules a limited number of regions for CP at each time. The objective is to design a scheduling algorithm that maximizes overall perception performance.
    
    \begin{figure}[!t]
    \centering
    \includegraphics[width=0.48\textwidth]{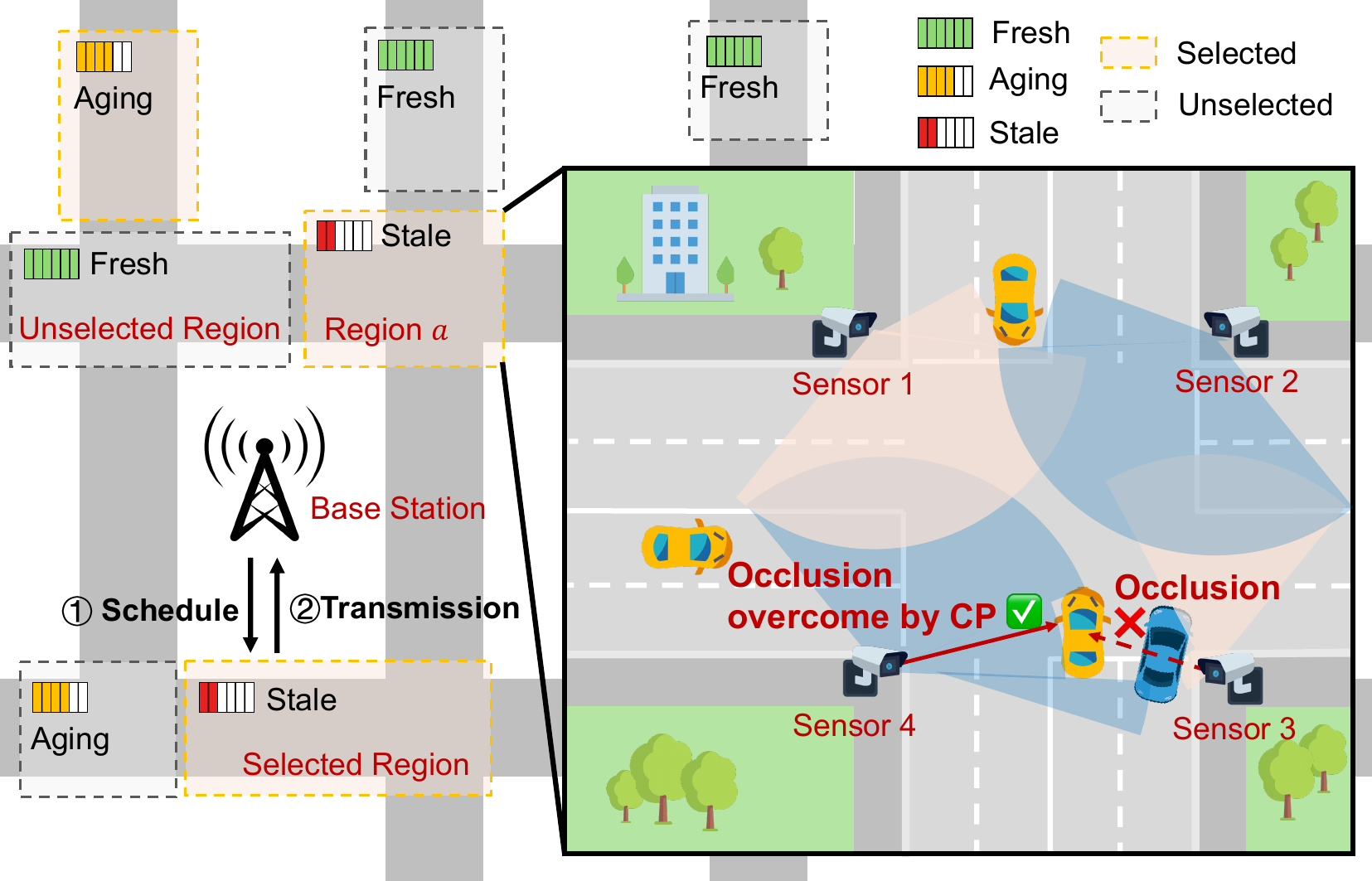}
    % \vspace{0mm}
    \caption{An illustration of the system architecture, where a BS manages multiple CP regions.}
    \label{system1}
    % \vspace{-3mm}
    \end{figure}

    This system operates in discrete slots, indexed by $k$, and entire process is orchestrated by the BS.
    The system workflow can be broken down into three sequential phases: 1) region selection and feature extraction, 2) bandwidth allocation and feature transmission, and 3) feature fusion and detection.
    
    In the first phase, the BS performs \emph{region selection}. The BS chooses regions $\mathcal{A}_{\text{selected},k}$ from the idle regions $\mathcal{A}_{\text{idle},k}$ for scheduling and adds them to the set of active regions, denoted by $\mathcal{A}_k \subseteq \mathcal{A}$. A region is active if its CP task has been started but is not yet completed. Thus, 
    $\mathcal{A}_k = (\mathcal{A} \setminus \mathcal{A}_{\text{idle},k}) \cup \mathcal{A}_{\text{selected},k}$.
    Following the scheduling decisions, for sensors in each selected region $a \in \mathcal{A}_{\text{selected},k}$, start to collect raw perceptual data from the environment and extract features.
    In the second phase, for each sensor in active region $a \in \mathcal{A}_k$, the BS conducts \emph{resource allocation}, assigning bandwidth $B_{n,k}$ to sensor $n$ for data transmission. The features from sensors are then compressed to fit the allocated bandwidth and are subsequently transmitted to the BS via wireless channels.
    In the third phase, the BS fuses the received features from all sensors within each selected region. The fusion result is then used for downstream tasks, such as object detection and tracking.
    
    \subsection{Feature Extraction Model}
    When the BS schedules region $a$ in slot $k$, we assume the delay for broadcasting the control message to all sensors in that region is negligible. Upon receiving control message, each sensor $n \in \mathcal{N}_a$ processes its raw sensing data to produce an extracted feature, denoted as $\mathcal{F}_n^{\text{ext}}$. 
    To manage the transmission and computing resources consumed per region, we impose a long-term average constraint on the communication volume for each region $a$:
    \begin{equation} \limsup_{K \to \infty} \frac{1}{K} \sum_{k=1}^{K} \mathbb{E}\left[b_{a,k}\right] \leq \Gamma_a, \quad \forall a \in \mathcal{A}, 
    \end{equation}
    where $\Gamma_a$ is the predefined communication volume budget for region $a$. 
    Notably, the BS allocates communication volume $b_{a,k}$ to each region $a$ based on channel conditions and the current state of region $a$. 
    % Individual sensors operate under constraints such as finite energy budgets. The energy consumed for data transmission is strongly correlated with the number of bits transmitted \cite{heinzelman2000energy, lindsey2002pegasis}.
    The extracted feature $\mathcal{F}_n^{\text{ext}}$ from sensor $n \in \mathcal{N}_a$ is compressed on-demand into a transmittable feature $\mathcal{F}_n^{\text{tr}}$ that matches the allocated communication volume $b_{a,k}$. Let $\left\lvert \cdot \right\rvert$ represents the data size of feature. Thus, $\sum_{n \in \mathcal{N}_a} \left\lvert \mathcal{F}_n^{\text{tr}} \right\rvert=b_{a,k}$. The communication volume allocation and feature compression method for sensors within each region is detailed in Section \ref{scheduling algorithm}.
    
    The feature extraction latency for an individual sensor $n$, denoted by $d_{n,k}^{\text{ext}}$, is modeled as a random variable following a shifted exponential distribution \cite{wu2020latency, zhang2021coded, sun2022coded}. The subsequent compression delay is considered negligible. Since all sensors operate in parallel, the overall phase delay for region $a$ is dictated by the sensor that finishes last. Therefore, it is expressed as:
    \begin{equation}
        d_{a,k}^{\text{ext}}=\max_{n \in \mathcal{N}_a}\left\{d_{n,k}^{\text{ext}}\right\}.
    \end{equation}

    \subsection{Feature Transmission Model}
    Following extraction, the features are transmitted to the BS. The achievable transmission rate of sensor $n$ at slot $k$, denoted by $r_{n,k}$, is subject to the spectral efficiency $\eta_k$ (in b/s/Hz), and is given by:
    \begin{equation} 
        r_{n,k} = B_{n,k} \cdot  \eta_k, \label{rate} 
    \end{equation}
    where $B_{n,k}$ is the bandwidth allocated to sensor $n$. The spectral efficiency $\eta_k$ is a variable determined by the channel state, and we assume it remains constant within each task.
    % $P_{n}^{\text{tx}}$ is its transmit power, $G_{n,k}$ denotes the channel gain, and $\sigma^2$ represents the noise power.

    Given the total wireless bandwidth $B_\text{total}$, we need to distribute it efficiently among the sensors in the active regions $\mathcal{A}_k$. We jointly determine the communication volume $b_{n,k}$ and the bandwidth $B_{n,k}$ of sensor $n$. This decision is based on the importance of the sensor data and their correlation. The specific algorithm for this allocation is detailed in Section \ref{scheduling algorithm}. The sum of the allocated bandwidth must satisfy:
    \begin{equation}
    \sum_{a \in \mathcal{A}_k} \sum_{n \in \mathcal{N}_a} B_{n,k} \leq B_\text{total}.
    \end{equation}
    The transmission delay for region $a$, denoted by $d_{a,k}^{\text{tr}}$ is determined by the slowest sensor and is given by:
    \begin{equation}
    d_{a,k}^{\text{tr}} = \max_{n \in \mathcal{N}_a} \left\{\frac{b_{n,k}}{r_{n,k}}\right\}.    
    \end{equation}
    
    \subsection{Feature Detection Model}
    Once the features are successfully uploaded, the BS performs feature fusion and processing. First, the BS fuses the features received from all active sensors in region $a$. Let $\mathcal{F}_{n,k}^{\text{tr}}$ be the feature set from sensor $n \in \mathcal{N}_a$. The fusion process aggregates these individual sets into a comprehensive regional feature set, $\mathcal{F}_{a,k}^{\text{tr}} = \bigcup_{n \in \mathcal{N}_a} \mathcal{F}_{n,k}^{\text{tr}}$. This fused feature set is then used for downstream tasks, such as object detection. The feature processing delay for the task of region $a$, started at slot $k$, denoted by $d_{a,k}^{\text{det}}$, is modeled as a random variable following a shifted exponential distribution \cite{wu2020latency, zhang2021coded, sun2022coded}. 
    % However, the BS is constrained by its finite computational capacity, particularly its GPU resources. This limits the number of parallel processing tasks, meaning the BS can simultaneously schedule and process features from at most $M$ regions.
    However, the BS is constrained by its computational resources (e.g., GPU capacity), which limits the number of CP tasks that can be active concurrently. Since $\mathcal{A}_k$ is the set of active regions with ongoing tasks in slot $k$, its size is upper-bounded by $M$. This imposes the following system constraint: 
    \begin{equation} 
    \left\lvert \mathcal{A}_k \right\rvert \leq M, \quad \forall k. 
    \end{equation}

    The physical delay for the CP task in region $a$ initiated at time $k$, denoted by $d_{a,k}^{\text{sec}}$ (in seconds), is defined as the sum of the constituent delays from the three sequential phases: extraction, transmission, and detection:
    \begin{equation}
        d_{a,k}^\text{sec} \triangleq d_{a,k}^{\text{ext}} + d_{a,k}^{\text{tr}} + d_{a,k}^{\text{det}}.
    \end{equation}
    Let $\tau$ be the slot length. To ensure consistency with the discrete slot model, the final task delay $d_{a,k}$ (in slots) is calculated by rounding the physical delay up to the nearest integer:
    \begin{equation}
        d_{a,k} = \left\lceil \frac{d_{a,k}^\text{sec}}{\tau} \right\rceil.
    \end{equation}

    \subsection{Timeliness Metric for CP}
    The timeliness metric is jointly affected by region AoI and the communication volume of each sensor. 
    The AoI of region $a$ in slot $k$, denoted by $h_{a,k}$, represents the time elapsed since the data for the last successfully completed CP task was generated. The evolution of AoI is illustrated in Fig.~\ref{aoi_penalty}.
    % where the subscript $a$ is omitted for clarity. The AoI $h_k$ increases in a staircase manner over time and is reset only upon the successful completion of a CP task. 
    If region $a$ completes a CP task upon slot $k=k'_m$, the AoI $h_{a,k}$ is reset to the total task delay $d_{a,k_m}$.
    The AoI evolves according to the following dynamics:
    \begin{equation}
    h_{a,k} =
    \begin{cases}
        d_{a,k_m},   & \text{if } k = k_m', \\
        h_{a,k-1} + 1, & \text{otherwise}.
    \end{cases} \label{eq:aoi}
    \end{equation}
    
    Let $\bm{b}_{a,k}=\left\{b_{n,k} \mid n \in \mathcal{N}_a\right\}$ denote the allocated communication volumes for the sensors in region $a$ in slot $k$. We introduce a penalty function $f_a(h_{a,k}, \bm{b}_{a,k})$ as timeliness metric. The dynamics of this penalty function are illustrated in Fig.~\ref{aoi_penalty}. 
    % For simplicity, we consider a single region and omit the subscript $a$.
    We define the $m$-th scheduling as the one that initiates the CP task in slot $k_m$ and finishes in slot $k'_m$. We define the $m$-th interval as the period between the completion of two consecutive tasks, from slot $k_m'$ to $k_{m+1}'$. Since the communication volume allocated for a task only affects the perception performance after that task is completed, the performance of the $m$-th interval is affected by the communication volume $\bm{b}_{k_m}$, allocated at the $m$-th scheduling instant.
    The function $f(h,\bm{b})$ will be fitted on the real-world roadside CP dataset in section \ref{empirical}, and is characterized by two properties.
    First, it is a non-decreasing function of the AoI $h$. This reflects that more stale information results in a higher penalty. 
    Second, it is a non-increasing function of the communication volume $\bm{b}$. This represents that transmitting a larger data volume generally leads to higher perception quality, thus incurring a lower penalty.
    This penalty function formalizes a fundamental trade-off. On one hand, scheduling frequent updates with large data volumes reduces the penalty. On the other hand, this approach consumes significant communication and computational resources, which in turn leads to increased delay. Therefore, the objective is to design a scheduling algorithm that manages this trade-off to minimize the long-term average penalty.
    \begin{figure}[!t]
        \centering
        \includegraphics[width=0.48\textwidth]{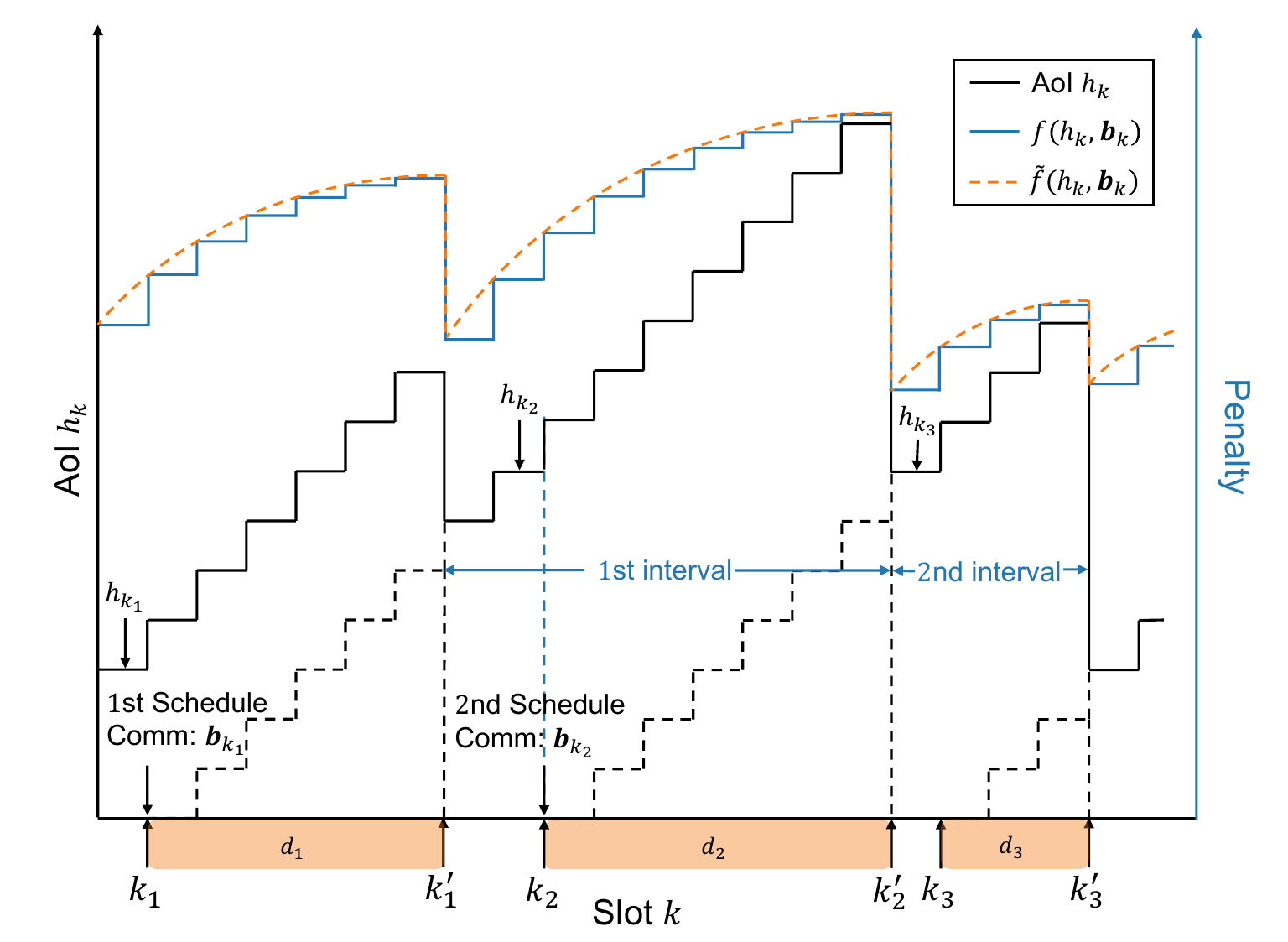}
        % \vspace{0mm}
        \caption{Timeliness-oriented penalty function with AoI and the communication volume in a region. The AoI $h_k$ increases in a staircase manner over time and is reset only upon the completion of a CP task.
        The $m$-th schedule begins at slot $k_m$ with the allocated communication volume (Comm) $\bm{b}_{k_m}$, and the task is completed upon slot $k'_m$.
        The $m$-th interval is defined as the duration between the completion time of two tasks, from $k'_m$ to $k'_{m+1}$. Note that $h_{k_m}$ indicates the AoI at the $m$-th scheduling time.
        }
        \label{aoi_penalty}
        % \vspace{-3mm}
    \end{figure}

    \subsection{Problem Formulation}
    Our objective is to design a scheduling algorithm to minimize the long-term average penalty across all regions.
    Recall that the scheduling decisions are made at the beginning of each slot, while the CP task delay may span one or multiple slots.
    We define three scheduling decision variables in each slot $k$.  
    The first is the \emph{region scheduling vector}, $\bm{u}(k) \triangleq [u_{1,k}, \dots, u_{A,k}]$, where $u_{a,k} \in \{0, 1\}$ is a binary indicator equal to $1$ if region $a$ is scheduled in slot $k$, and $0$ otherwise. 
    The second is the \emph{bandwidth allocation matrix}, $\bm{B}(k) \triangleq [\bm{B}_{1,k}, \dots, \bm{B}_{A,k}]^T$, where $\bm{B}_{a,k}=\left\{B_{n,k} \mid n \in \mathcal{N}_a\right\}$ represents the bandwidth allocated to sensors in region $a$.
    The third is the \emph{communication volume allocation matrix}, $\bm{b}(k) \triangleq [\bm{b}_{1,k}, \dots, \bm{b}_{A,k}]^T$. Recall that $\bm{b}_{a,k}=\left\{b_{n,k} \mid n \in \mathcal{N}_a\right\}$ is the communication volume of sensors in region $a$, while $b_{n,k} \in [b_{\min}, b_{\max}]$, with $b_{\min}$ and $ b_{\max}$ representing the minimum and maximum. 
    Let $b_{a,k}$ be the communication volume of region $a$ in slot $k$, i.e., $b_{a,k} = \sum_{n \in \mathcal{N}_a} b_{n,k}$.
    Let $\bar{b}_a$ be the long-term average communication volume:    
    \begin{equation}
        \bar{b}_a \triangleq \limsup_{K \to \infty} \frac{1}{K}  \mathbb{E} \left[\sum_{k=1}^{K} b_{a,k} \right].
    \end{equation}
    Let $\bar{f}_a$ be the long-term average penalty for region $a$:
    \begin{equation}
        \bar{f}_a \triangleq \limsup_{K \to \infty} \frac{1}{K}  \mathbb{E} \left[\sum_{k=1}^{K} f_a(h_{a,k}, \bm{b}_{a,k}) \right].
    \end{equation}
    Using this notation, the problem can be expressed as $\mathcal{P}1$. The objective \eqref{eq:objective} is to minimize the sum of the long-term average penalties over all regions. Constraint \eqref{eq:region_constraint} imposes a long-term average communication volume budget $\Gamma_a$ on each region $a$. Constraints \eqref{bandwidth_bound} and \eqref{const:concur_sched} are state-based, limiting the instantaneous bandwidth usage and the total number of currently active regions in the state set $\mathcal{A}_k$. The decision $u_{a,k}=1$ initiates a new CP task, whereupon region $a$ is added to the set $\mathcal{A}_k$, where it remains until the task is completed. Finally, constraints \eqref{binary_var} and \eqref{con_var} define the binary and continuous domains for the scheduling and communication volume allocation variables.   
    \begin{subequations} \label{eq:problem_P1}
    \begin{align}
    \mathcal{P}1: \min_{\{\bm{u}(k), \bm{b}(k) , \bm{B}(k)\}} &~ \sum_{a \in \mathcal{A}} \bar{f}_a \label{eq:objective} \\
    \text{s.t.~~~~} \quad & \bar{b}_a \le \Gamma_a,  \forall a \in \mathcal{A}, \label{eq:region_constraint} \\
    & \sum_{a \in \mathcal{A}_k} \sum_{n \in \mathcal{N}_a} B_{n,k} \leq B_\text{total}, \forall k,
    \label{bandwidth_bound} \\
    & \left\lvert \mathcal{A}_k \right\rvert \leq M, \forall k,
    \label{const:concur_sched} \\
    & u_{a,k} \in \{0, 1\}, \forall a \in \mathcal{A}, \forall k, \label{binary_var} \\
    & b_{n,k} \in [b_{\min}, b_{\max}], \forall n \in \mathcal{N}, \forall k. \label{con_var}
    \end{align}
    \end{subequations}
    
    Solving Problem $\mathcal{P}1$ is challenging due to several intertwined difficulties.
    First, the problem is \emph{stochastic} and involves both a \emph{long-term average} objective and constraint, requiring online decisions without prior knowledge of future system states (e.g., channel conditions). Satisfying these requirements is inherently difficult with slot-by-slot decisions.
    Second, scheduling impact is \emph{delayed and cumulative}. The benefit of an update accrues over time, making instantaneous performance a poor long-term quality indicator.
    Third, the region selection decision is \emph{combinatorial}. Since only a limited number of regions can be served simultaneously, there are numerous possible combinations, necessitating a concise method for selecting the optimal subset.
    Therefore, we leverage Lyapunov optimization theory to transform this intractable stochastic problem into a sequence of tractable, deterministic problems solved in each slot, enabling the design of a scheduling index to prioritize the region scheduling strategy.

    \section{Timeliness-Aware Multi-region Prioritized Scheduling Algorithm} \label{scheduling algorithm}
    In this section, we first establish the theoretical analysis on the long-term average penalty using an oracle perspective, which yields the design principle for our policy. We then leverage the Lyapunov optimization framework to derive the actionable priority metric, balancing scheduling worth against resource cost. Finally, we formalize the complete TAMP scheduling algorithm, detailing the competitive selection and resource allocation procedures. 
    
    \subsection{Average Penalty Analysis} \label{avg_pnt_anls}
    We analyze the fundamental performance trade-off of the system from an oracle-based perspective. We consider an oracle that possesses complete knowledge of the statistical properties of all random processes in the system (e.g., the probability distributions of channel conditions). Based on this statistical information, an optimal scheduling algorithm can be derived. 
    We focus on a single region and omit the subscript $a$ for simplicity. 
    To analyze long-term penalty analysis, we define a cumulative penalty function $F(h,\bm{b})$ as the sum of instantaneous penalties up to an AoI of $h$:
    \begin{equation} \label{eq:F_def}
        F(h,\bm{b}) \triangleq \sum_{x=0}^{h} f(x,\bm{b}).
    \end{equation}
    
    Since the instantaneous penalty function $f(h,\bm{b})$ is only defined over a discrete domain, we construct its continuous interpolation, denoted as $\Tilde{f}(x,\bm{b})$, such that $\Tilde{f}(x,\bm{b})=f(x,\bm{b})$ for all integer values of $x$. Given that $f(x,\bm{b})$ is non-negative and $\Tilde{f}(x,\bm{b})$ is non-decreasing w.r.t. $x$, the integral is bounded above by the right Riemann sum. Thus, we have:
    \begin{equation}
        \sum_{x=0}^{h} f(x,\bm{b}) \ge \int_0^h \Tilde{f}(x,\bm{b})dx.
    \end{equation}
    % Define the continuous cumulative penalty function $\Tilde{F}(x,\bm{b})=\int_0^h \Tilde{f}(x,\bm{b})dx$, we have $F(h,\bm{b}) \ge \Tilde{F}(h,\bm{b})$.
    Denoting the continuous cumulative penalty as $\Tilde{F}(h,\bm{b})=\int_0^h \Tilde{f}(x,\bm{b})dx$, we thus have $F(h,\bm{b}) \ge \Tilde{F}(h,\bm{b})$.

    We consider a scheduling policy $\pi$ that admits a stationary regime. Under this policy, let $h$ and $d$ be the random variables for the AoI at the scheduling time and the corresponding task delay, with averages $\bar{h} = \mathbb{E}[h]$ and $\bar{d} = \mathbb{E}[d]$. 
    The following lemma first provides an equivalent long-term average penalty of $\mathcal{P}1$. It then establishes a tractable approximation, derived by approximating the staircase penalty and applying Jensen's inequality, which serves as our objective for subsequent optimization. An interval is defined as the duration between the completion time of two consecutive tasks, as illustrated in Fig.~\ref{aoi_penalty}. Let $M_K$ be the number of intervals up to slot $K$.
    % The following lemma provides a approximate objective on the long-term average penalty of $\mathcal{P}1$.
    \begin{lemma} \label{lem1}
    Given scheduling policy $\pi$, the long-term average penalty for a single region is:
    \begin{align}
    & \limsup_{K\to\infty} \frac{1}{K} \mathbb{E}_\pi \left[ \sum_{k=0}^{K} f(h_k, \bm{b}_k) \right] \\
    &= \limsup_{K\to\infty} \frac{M_K}{K} \mathbb{E}_\pi [F(h+d,\bm{b}) - F(d,\bm{b})] \label{eq:longterm_penaty}\\
    &\geq \frac{1}{\bar{h}} \cdot \left[ \Tilde{F}(\bar{h}+\bar{d}, \bm{b}) - \mathbb{E}_\pi[F(d,\bm{b})] \right].
    \label{eq:final_bound}
    \end{align}
    \end{lemma}
    \begin{proof}
    See Appendix~\ref{sec:appendix_a}.
    \end{proof}
    While Lemma \ref{lem1} provides a tractable approximation for long-term average penalty, our online algorithm requires \emph{practical, per-slot} decisions. This decision process is decoupled into two steps.
    % first, determining the optimal communication volume if a CP task is scheduled, and second, deciding whether to initiate the schedule.
    In the first step, the BS assumes a schedule will occur in slot $k$ and determines \emph{the optimal communication volume} $\bm{b}_k^*$ for the current AoI $h_k$ and channel condition. This is obtained by solving the following per-slot penalty minimization problem guided by \eqref{eq:final_bound}:
    \begin{equation}
        \bm{b}_k^* = \arg\min_{\bm{b} \geq 0} \; \frac{1}{h_k} \left[ \Tilde{F}(h_k + d_k, \bm{b}) - F(d_k, \bm{b}) \right]. \label{eq:comm_decision}
    \end{equation}
    In the second step, given the candidate volume $\bm{b}_k^*$, the BS must decide \emph{whether to schedule} the region in slot $k$.
    To establish a criterion for this decision, we analyze the properties of \eqref{eq:final_bound} to characterize the attributes of an optimal scheduling instant. 
    Specifically, to find the long-term optimal scheduling threshold, for the candidate communication volume $\bm{b}$, we optimize the AoI $h$ using the average delay $\bar{d}$:
    \begin{equation}
        \mathcal{P}2: \quad \min_{h > 0} \quad \frac{1}{h} \left( \Tilde{F}(h+\bar{d}, \bm{b}) - \mathbb{E}[F(d,\bm{b})] \right)
    \end{equation}
    % The quasi-convexity of this problem is established by the following lemma.
    \begin{remark}
        The first step represents an instantaneous optimal decision based on the current state $h_k$ and $d_k$, while the second step is formulated to derive the scheduling guidance for the long-term strategy, thus the average value $\bar{d}$ is employed to characterize the policy attribute.
    \end{remark}
    \begin{lemma} \label{lem:convex}
    $\mathcal{P}2$ is a quasi-convex optimization problem.
    \end{lemma}
    \begin{proof}
    See Appendix~\ref{sec:appendix_b}.
    \end{proof}
    
    For a quasi-convex problem, the Karush-Kuhn-Tucker (KKT) condition is sufficient for global optimality.
    By analyzing the KKT condition of $\mathcal{P}2$, we derive a metric that captures the scheduling urgency. 
    The stationarity condition requires that the derivative of the objective function in $\mathcal{P}2$ with respect to $h$ must be zero:
    \begin{equation}
        \frac{1}{h^2} \cdot \left[ h\Tilde{f}(h+\bar{d}, \bm{b}) - (\Tilde{F}(h+\bar{d}, \bm{b}) - \mathbb{E}[F(d,\bm{b})]) \right] = 0. \label{eq:kkt}
    \end{equation}
    
    Motivated by the optimality condition, we define an index $U(h,b)$ that quantifies the utility of scheduling, defined as the expression on left-hand side of \eqref{eq:kkt}:
    \begin{equation}
        \!U(h, \!\bm{b}) \! \triangleq  \!\frac{1}{h^2} \!\left[ h \Tilde{f}(h+\bar{d}, \!\bm{b})\!-\! \left( \!\Tilde{F}(h+\bar{d}, \!\bm{b}) \! - \! \mathbb{E}[F(d,\!\bm{b})] \!\right) \!\right]\!\!. \label{eq:utility_index}
    \end{equation}
    For typical penalty functions, $U(h, \bm{b})$ is a non-decreasing function of the AoI $h$, which aligns with the intuition that the urgency to schedule an update should increase as information becomes more stale. This index also reveals the fundamental trade-off with respect to the communication volume $\bm{b}$, as an increase in $\bm{b}$ reduces the penalty $\Tilde{f}(h,\bm{b})$ but increases the delay $\bar{d}$.
    This utility index forms the basis of our priority metric for the multi-region scheduling algorithm.

    \begin{algorithm}[!t]
    \caption{The TAMP Scheduling Algorithm}
    \label{alg:tamp}
    \begin{algorithmic}[1]
        \State \textbf{Input:} Maximum scheduled regions $M$, bandwidth $B_{\text{total}}$, communication budget $\bm{\Gamma}$, control parameter $\bm{V}$.
        \State \textbf{Initialize:} $Q_{a,0} \gets 0$, $h_{a,0} \gets 1$, $\mathcal{A}_{\text{idle},0} \gets \mathcal{A}$.
            \For{each slot $k=1, 2, \dots$}
            % \Statex \textit{// 1. Candidate Evaluation}
            \State Initialize $\mathcal{C} \gets \emptyset$.
            \Comment{Candidate Evaluation}
            \ForAll{idle region $a \in \mathcal{A}_{\text{idle},k}$}
                \State Calculate optimal volume $b_{a,k}^*$ using \eqref{eq:comm_decision}.
                \State $\Pi_{a,k} \gets U_a(h_{a,k}, b_{a,k}^*) - V_a Q_{a,k} b_{a,k}^*$.
                \If{$\Pi_{a,k} \ge 0$}
                    \State Add $(a, \Pi_{a,k})$ to $\mathcal{C}$.
                \EndIf
            \EndFor
            % \Statex \textit{// 2. Competitive Selection} 
            \State $M_{\text{rem}} \gets M - |\mathcal{A} \setminus \mathcal{A}_{\text{idle},k}|$.
            \Comment{Competitive Selection}
            \State Sort $\mathcal{C}$ by $\Pi_{a,k}$ in descending order.
            \State $\mathcal{A}_{\text{selected},k} \gets \text{Top}(\mathcal{C}, \min(M_{\text{rem}},|\mathcal{C}|))$.
            \State $\mathcal{A}_{k} \gets (\mathcal{A} \setminus \mathcal{A}_{\text{idle},k}) \cup \mathcal{A}_{\text{selected},k}$.    
            % \Statex \textit{// 3. Resource Allocation}
             \ForAll{region $a \in \mathcal{A}$}
                \Comment{Resource Allocation}
                \If {$a \in \mathcal{A}_{\text{selected}, k}$}
                    \State $u_{a,k} \gets 1$, $b_{a,k} \gets b_{a,k}^*$.
                    \State $\{b_{n,k} \mid n \in \mathcal{N}_a\} \gets \text{Split}(b_{a,k})$.
                \Else
                    \State $u_{a,k} \gets 0$, $b_{a,k} \gets 0$.
                \EndIf
                \If{$a \in \mathcal{A}_{k}$}
                    \State $B_{a,k} \gets B_{\text{total}} / M$.
                \EndIf
            \EndFor
            % \Statex \textit{// 4. System State Update}
            \ForAll{region $a \in \mathcal{A}$}
            \Comment{System State Update}
                \State Update virtual queue $Q_{a,k+1}$ using \eqref{eq:virtual_queue}.
                \State Update AoI $h_{a,k+1}$ using \eqref{eq:aoi}.
            \EndFor
            \State Update the idle region set $\mathcal{A}_{\text{idle},k+1}$.
        \EndFor
    \end{algorithmic}
    \end{algorithm}  
    
    \subsection{Scheduling Priority} \label{sche_prio}
    The previous subsection introduced  a scheduling utility index, $U_a(h_a,\bm{b}_a)$, quantifies scheduling utility of each region. We now develop a metric balancing this utility against the long-term communication constraint.
    % Our proposed online strategy is based on the framework of Lyapunov optimization. This framework allows us to convert a long-term average constraint into a per-slot objective by introducing a virtual queue.
    Empirical analysis in Section \ref{empirical} under the Where2comm framework \cite{hu2022where2comm} shows that CP performance is sensitive to the total allocated communication volume of each region, which it dynamically distributes communication volume among sensors according to the importance of their fields of view.
    Accordingly, we assign a maximum allowed long-term average communication volume $\Gamma_a$ to each region $a \in \mathcal{A}$. 
    To meet this budget, we introduce a virtual queue, $Q_{a,k}$ for each region. This queue tracks the ``communication deficit'' for region $a$ and evolves as follows:
    \begin{equation}
        Q_{a, k+1} = \max \left\{ Q_{a,k} + b_{a,k} - \Gamma_a, 0 \right\}, \label{eq:virtual_queue}
    \end{equation}
    where $b_{a,k}$ is the communication volume allocated to region $a$ in slot $k$. Note that $b_{a,k}=0$ if region $a$ is not scheduled in slot $k$.
    Intuitively, if the communication usage $b_{a,k}$ exceeds the budget $\Gamma_a$, the queue $Q_{a,k}$ grows, signaling the system is over-budget. A large $Q_{a,k}$ indicates a strong need to conserve communication resources in subsequent slots. A stable virtual queue enforces the long-term communication budget $\Gamma_a$.

    Our proposed online strategy is based on the framework of Lyapunov optimization, and the core is the drift-plus-penalty principle \cite{neely2010stochastic}. The goal in each slot is to maximize scheduling utility while pushing the virtual queue towards zero.
    Define the quadratic Lyapunov function $L(Q_{a,k})=\frac{1}{2}Q_{a,k}^2$, and the one-slot conditional drift is:
    \begin{equation}
        \Delta(Q_{a,k}) = \mathbb{E}[L(Q_{a,k+1})-L(Q_{a,k})\mid Q_{a,k}].
    \end{equation}
    The drift-plus-penalty expression is:
    \begin{equation}
        \Delta(Q_{a,k}) - V_a \, \mathbb{E}[ U_a(h_{a,k}, b_{a,k}^*) \, u_{a,k} \mid Q_{a,k} ], 
    \end{equation}
    where $V_a$ is a non-negative parameter controlling the drift and penalty trade-off. 
    To ensure tractability, we minimize an upper bound on the drift-plus-penalty expression. This transforms the long-term problem into a deterministic, per-slot problem, as formalized in the following Theorem~\ref{theorem:dpp}.
    \begin{theorem} \label{theorem:dpp}
    By minimizing an upper bound on the drift-plus-penalty expression, the scheduling algorithm is transformed into a deterministic per-slot problem, equivalent to selecting the scheduling action $u_{a,k} \in \{0,1\}$ that solves the following maximization problem in each slot $k$:
    \begin{equation}
         \max_{u_{a,k} \in \{0,1\}} \quad \left[ U_a(h_{a,k}, b_{a,k}^*) - V_a \cdot Q_{a,k} \cdot b_{a,k}^* \right] u_{a,k}.
        \label{eq:per_slot_objective}
    \end{equation}
    \end{theorem}
    \begin{proof}
    See Appendix~\ref{sec:appendix_c}.
    \end{proof}
    % The solution to \eqref{eq:per_slot_objective} is straightforward: set $u_{a,k}=1$ if and only if $U_a(h_{a,k}, b_{a,k}^*) - V_a \cdot Q_{a,k} \cdot b_{a,k}^* > 0$. This yields a simple threshold rule: trigger an update when the priority index minus the virtual-queue penalty is positive, and otherwise remain idle. The parameter $V$ directly tunes the trade-off between immediate performance and long-term communication budget compliance.
    The objective \eqref{eq:per_slot_objective} inspires our final scheduling priority score $\Pi_{a,k}$, for each region $a$ at slot $k$:
    \begin{equation}
    \Pi_{a,k} = U_a(h_{a,k}, b_{a,k}^*) - V_a \cdot Q_{a,k} \cdot b_{a,k}^*.
    \label{eq:final_priority_score}
    \end{equation}
    Here, $b_{a,k}^*$ is the optimal communication volume for region $a$ at slot $k$, determined by solving \eqref{eq:comm_decision}.
    \begin{remark}
    The priority score $\Pi_{a,k}$ elegantly trades off scheduling utility and communication cost.
    The first term, $U_a(\cdot)$, represents the \emph{scheduling worth}, derived from our penalty analysis.
    The second term, $V_a Q_{a,k} b_{a,k}^*$, represents the \emph{scheduling cost}, penalizing excess communication cost.
    \end{remark}
    
    \subsection{Timeliness-Aware Multi-region Prioritized Scheduling} \label{TAMP_alg}
    Building upon the scheduling priority score $\Pi_{a,k}$, we now present our timeliness-aware multi-region prioritized (TAMP) scheduling algorithm for CP. 
    The core idea is a competitive selection process: in each slot, all idle regions are evaluated for scheduling priority.
    The TAMP algorithm then schedules the most deserving regions such that the total number of active regions does not exceed the system capacity $M$.
    The complete procedure is formalized in Algorithm~\ref{alg:tamp}.

    Given the maximum scheduled regions $M$, the total bandwidth $B_{\text{total}}$, the communication budget $\bm{\Gamma}= \left\{\Gamma_a \mid a \in \mathcal{A}\right\}$ and the control parameter $\bm{V}= \left\{V_a \mid a \in \mathcal{A}\right\}$, the algorithm unfolds in four stages within each slot:

    \emph{1) Candidate Evaluation:} The BS assesses each idle region by calculating its optimal communication volume $b_{a,k}^*$ and a corresponding priority score $\Pi_{a,k}$. Regions with a non-negative score are added to the candidate set $\mathcal{C}$.

    \emph{2) Competitive Selection:} These candidates are ranked in descending order by priority scores. The available scheduling capacity, denoted as $M_{\text{rem}}$, is calculated as the BS capability $M$ minus the number of currently active regions, $|\mathcal{A} \setminus \mathcal{A}_{\text{idle},k}|$. The BS then selects the top $\min(M_{\text{rem}}, |\mathcal{C}|)$ regions to initiate new tasks, merging them with existing active regions to form the final active set $\mathcal{A}_k$.

    \emph{3) Resource Allocation:} For the newly selected regions $a \in \mathcal{A}_{\text{selected},k}$, the scheduling variable $u_{a,k}$ is set to 1. The assigned regional volume $b_{a,k}$ is distributed to sensors using the $\text{Split}(\cdot)$ function, illustrated in Fig.~\ref{workflow}. This function utilizes a spatial confidence mask \cite{hu2022where2comm} and dynamically adjusts a confidence threshold $\theta$ via binary search, thereby assigning communication volumes to capture the most salient object features from view of each sensor.
    Specifically, $\theta$ is iteratively decreased (to include more features) or increased (to compress data) until the sum of the intermediate sensor volumes, $\sum_{n \in \mathcal{N}_a} b_n'$, aligns with the budget $b_{a,k}$ within a tolerance $\xi$. Subsequently, each region in the final active set $\mathcal{A}_k$ is allocated an equal share of the total bandwidth, $B_{\text{total}}/M$. 

    \emph{4) System State Update:} Finally, the AoI and virtual queues for all regions are updated according to their dynamics. Any region that has just completed its task is returned to the idle set $\mathcal{A}_{\text{idle}}$ for the next slot.
    
    \begin{figure}[!t]
    \centering
    \includegraphics[width=0.48\textwidth]{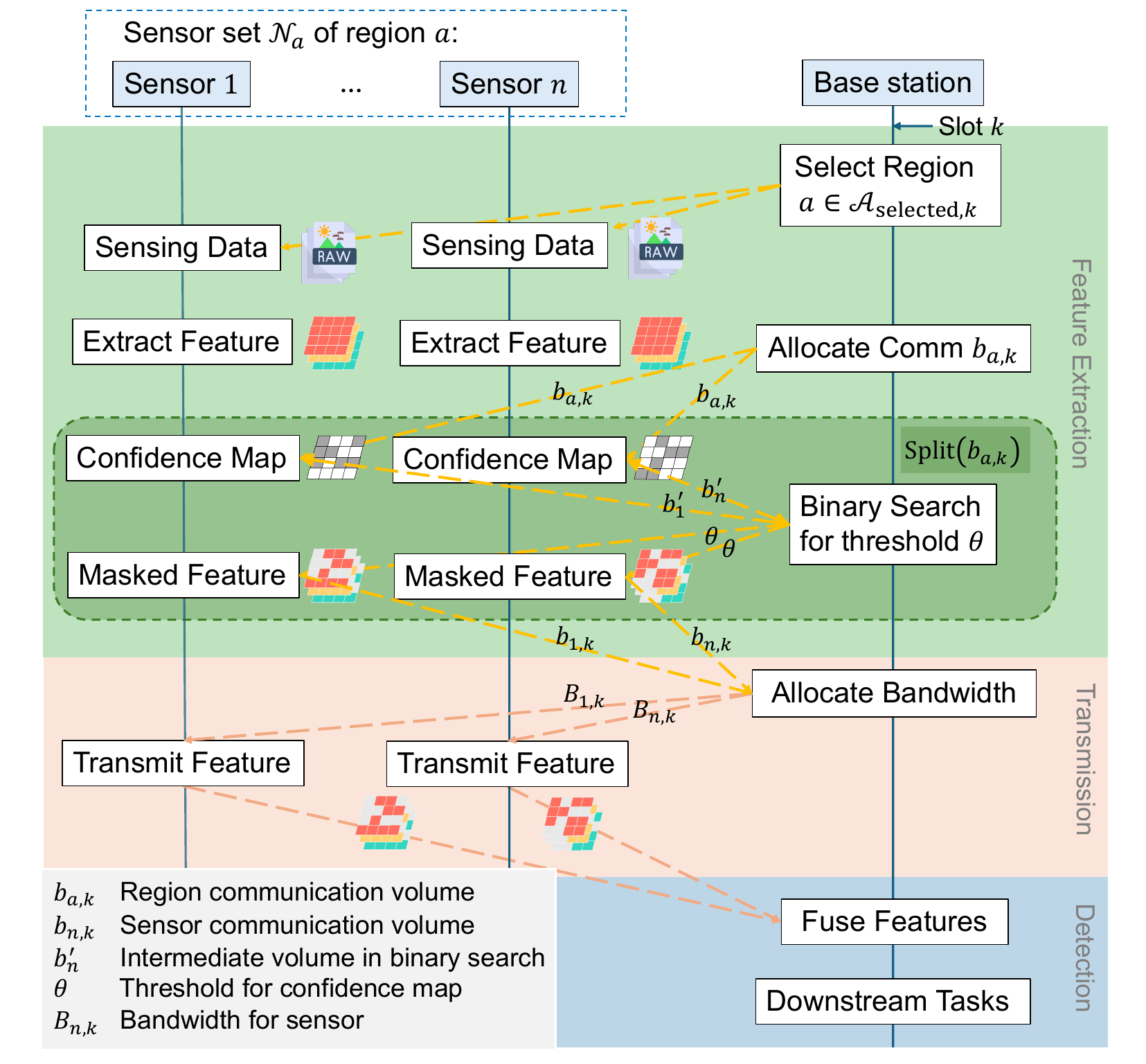}
    % \vspace{-6mm}
    \caption{The workflow of the multi-region CP system.}
    \label{workflow}
    % \vspace{-10mm}
    \end{figure}
    
    The TAMP scheduling algorithm, is executed within the three-phase workflow mentioned in Section \ref{system model}, illustrated in Fig.~\ref{workflow}. The four stages of algorithm map to the physical system as follows:
        \begin{itemize}
            \item Feature Extraction: The BS executes \emph{candidate evaluation} and \emph{competitive selection}, determining the selected regions $\mathcal{A}_{\text{selected},k}$ in slot $k$.
            For the first part of \textit{resource allocation}, the BS allocates communication volume $b_{a,k}^*$ to regions. Then, sensors generate the compressed feature with volume $b_{n,k}$ based on their spatial confidence maps.

            \item Feature Transmission: For the second part of \textit{resource allocation}, the BS allocates bandwidth $B_{a,k}$ to active regions $\mathcal{A}_k$, and $B_{n,k}$ to each sensor to equalize the transmission delay among all sensors within the region. Then sensors transmit compressed features to the BS.
            % $B_{n,k}=\frac{b_{n,k}}{b_{a,k}}\cdot B_{a,k}$
            
            \item Feature Fusion and Detection: The BS fuses the received features and conducts downstream tasks. Then the BS executes \textit{system state update}, refreshing AoI, virtual queues and the set of idle regions after task completion.
        \end{itemize}
    
    \section{Empirical Penalty Function} \label{empirical}
    In this section, we establish an empirical penalty function based on the RCooper dataset to bridge the theoretical scheduling framework with real-world CP performance. We model the relationships between Average Precision (AP), AoI, and communication volume in different scenarios, deriving utility metrics that guide the TAMP algorithm.
    
    \subsection{Experimental Setup}
    To empirically model the penalty function, we conducted experiments on the real-world RCooper dataset \cite{hao2024rcooper}, which features two roadside CP scenarios: intersection and corridor, as shown in Fig.~\ref{rcooper_scenario}. Each intersection is equipped with four sensors, consisting of two multiline LiDAR groups (one 80-beam and one 32-beam) and two MEMS LiDARs. 
    % The two sensors located on the diagonal are of the same type. 
    Each corridor is equipped with two multiline LiDAR groups (one 80-beam and one 32-beam).
    % To understand the relationship between performance, AoI, and communication volume, we conducted experiments on the RCooper dataset. 
    The dataset includes 246 sequences from corridor scenarios and 34 sequences from intersection scenarios. Each 15-second sequence is captured at 3 Hz, comprising approximately 45 point cloud frames.
    Our methodology is based on the Where2Comm CP framework \cite{hu2022where2comm}. We evaluate object detection performance using Average Precision (AP), where detections are matched to ground-truth objects based on the Intersection over Union (IoU). We evaluate AP at IoU thresholds of 0.3, 0.5, and 0.7.
    % Our experiment is based on the Where2Comm framework \cite{hu2022where2comm}, and we use AP to quantify object detection performance. We control the two key variables as follows:
    We systematically control the two variables that influence the penalty function:
    \begin{itemize}
        \item AoI: We simulate information staleness by introducing a temporal offset between a point cloud frame and its ground-truth label. For instance, an AoI of 0.2s is created by evaluating a  detection results frame against the labels from 0.2s prior.
        \item Communication Volume: We regulate the data volume using spatial-confidence-aware communication mechanism \cite{hu2022where2comm}. A spatial confidence map is generated, and by applying a threshold to this map, only features from high-confidence areas (e.g., those likely to contain objects) are selected for transmission, allowing us to control the total communication volume of each region.
    \end{itemize}
    \begin{figure}[!t]
    \centering
    \subfloat[Intersection Scenario]{
        \includegraphics[width=0.455\columnwidth]{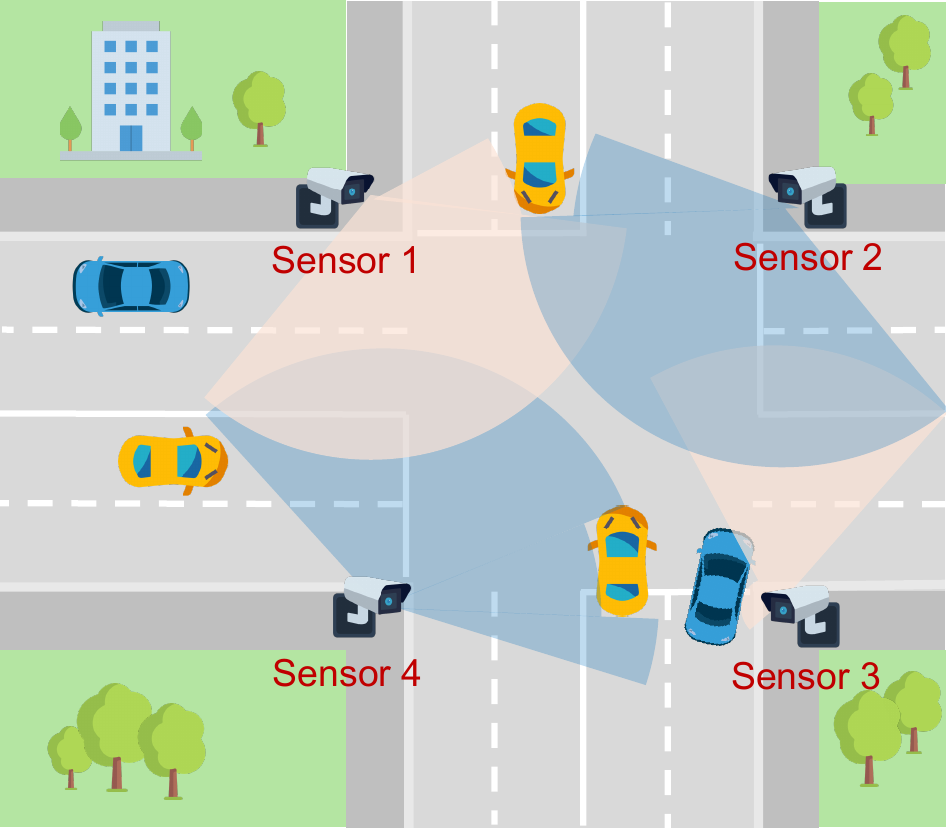}
        \label{fig:intersection}
    }
    \hfill
    \subfloat[Corridor Scenario]{
        \includegraphics[width=0.455\columnwidth]{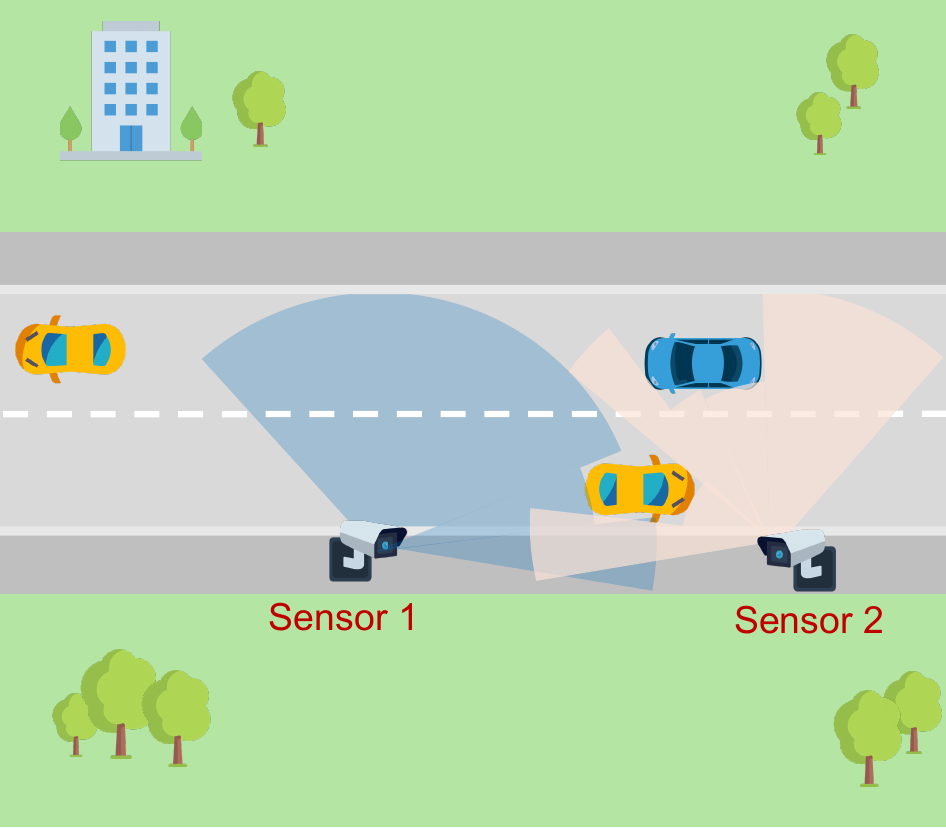}
        \label{fig:corridor}
    }
    \caption{Two typical roadside CP scenarios of RCooper dataset\cite{hao2024rcooper}.}
    \label{rcooper_scenario}
    \end{figure}

    \subsection{Performance Modeling and Penalty Function}
    Experimental data in Fig.~\ref{fig:intersection_aoi_comm} visualizes the joint influence of AoI and communication volume on AP in the intersection scenario.
    The AP initially degrades as the AoI increases. Subsequently, the degradation gradually flattens, approaching a stable AP value. This stable baseline is interpreted as the detection accuracy for static background objects, which remains largely unaffected by information staleness.
    Separately, additional communication volume often provides a sharp initial performance gain by revealing critical occluded objects, but this gain gradually levels off.
    The combined influence is modeled by the dual exponential model, and the AP, $\rho_1$, is given by the fitted surface in Fig.~\ref{fig:intersection_aoi_comm} and expressed as:
    \begin{equation}
    \rho_1(h_\text{s},b_\text{log}) = \alpha \cdot e^{-\beta \cdot h_\text{s}} - \gamma \cdot e^{-\delta \cdot b_\text{log}} + \epsilon,
    \label{eq:int_dual_exp_math}
    \end{equation}
    where $h_\text{s}$ is AoI in seconds, $b_\text{log}$ is the total communication volume of sensors in $\log_2(\text{Bytes})$, and $\alpha, \beta, \gamma, \delta$ and $\epsilon$ are non-negative model parameters.

    \begin{figure}[!t]
    \centering
    \includegraphics[width=0.48\textwidth]{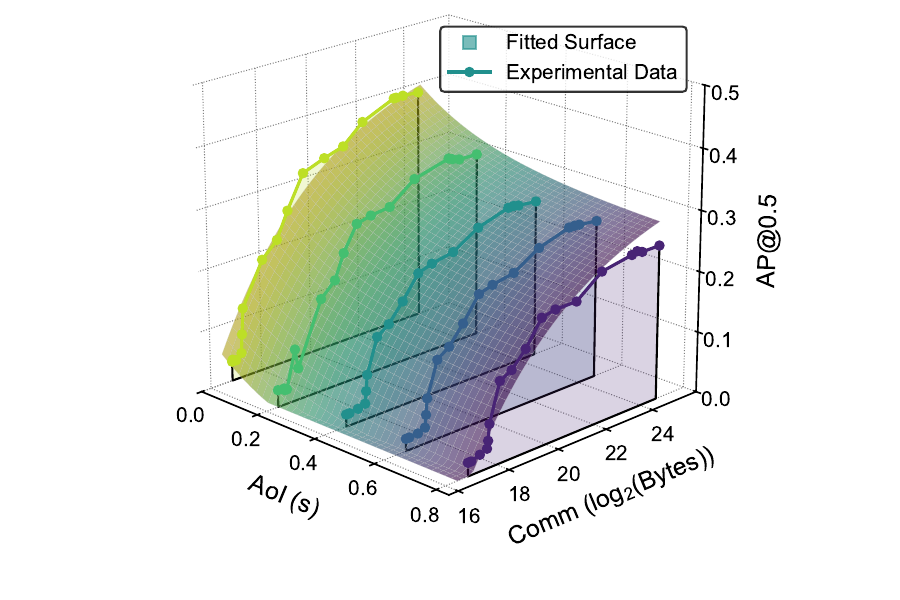}
    \vspace{-6mm}
    \caption{Performance fitting in the \textit{intersection} scenario, which depicts the joint impact of AoI and communication volume (in $\log_2(\text{Bytes})$) on AP@0.5.}
    \label{fig:intersection_aoi_comm}
    \end{figure}
    
    We now define the continuous penalty function, $\Tilde{f}(h,b)$, used by our scheduling algorithm, which represents the loss in AP relative to an ideal baseline. 
    % This loss is modeled as a function of the AoI and the communication volume.
    First, we define the ideal performance baseline, denoted as $\rho_\text{max}$. This value represents the maximum achievable AP at zero AoI and maximum communication volume, i.e., $\Tilde{f}(0,b_\text{max})=0$. The penalty function is formulated as the difference between the baseline and the performance modeled by our fitted empirical function:
    \begin{equation}
    \Tilde{f}(h, b) = \rho_\text{max} - \rho_1(h_\text{s},b_\text{log}).
    \label{penalty_int}
    \end{equation}
    The physical variables used in the fitted model are mapped from the decision variables in algorithm as follows:
    \begin{equation}
        h_\text{s} = h \cdot \tau,~~
        b_\text{log} = \log_2\left(\frac{b \cdot 10^6}{8}\right).
    \end{equation}
    Here, $\tau$ is the duration of a single slot, the AoI $h$ is measured in slots, and the communication volume $b$ is measured in Megabits (Mb). 
    Substituting the above variable transformation into \eqref{penalty_int} yields the penalty function:    
    \begin{equation}
        \Tilde{f}(h, b) = -\alpha \cdot e^{\beta h \tau} + \gamma \left( \frac{b\cdot 10^6}{8}\right)^{-\delta/\ln2}-\epsilon_0,
    \end{equation}
    where $\epsilon_0=\rho_\text{max}-\epsilon$. Furthermore, we have:
    \begin{equation}
        \Tilde{F}(h, b) 
        = -\frac{\alpha}{\beta \tau}(1-e^{-\beta \tau h})+h \! \cdot \! \left[\gamma \left( \frac{b\cdot 10^6}{8} \right)^{-\delta/\ln2} \!\!+\! \epsilon_0 \right]\!. \nonumber
    \end{equation}
    Applying this penalty function into the scheduling utility index $U(h,b)$ defined in Subsection \ref{avg_pnt_anls}, yields the utility metric used by our algorithm:
    \begin{equation}
        U(h,b)=-\frac{\alpha}{h}e^{-\beta h \tau}-\frac{\alpha}{\beta \tau h^2}e^{-\beta \tau \bar{d}}\left(e^{-\beta h \tau}-1\right),
    \end{equation}
    where $\bar{d}$ is the expected delay in this scenario.
    This derived expression directly guides the competitive scheduling decisions in the algorithm presented in Section~\ref{scheduling algorithm}.    

    % In the corridor scenario, as shown in Fig~\ref{fig:corridor_aoi_vs_ap}, the AP degrades more rapidly with increasing AoI compared to the intersection scenario, which is primarily attributed to the typically higher vehicle speeds in corridor environments. 
    In the corridor scenario, Fig.~\ref{fig:corridor_aoi_vs_ap} shows the AP degrades more rapidly with increasing AoI compared to the intersection scenario, which is due to the typically higher vehicle speeds in corridor scenario. This rapid performance degradation is mathematically captured using an exponential decay function.
    The experimental data in Fig.~\ref{fig:corridor_comm_vs_ap_aois} shows the influence of communication volume on AP under different AoIs.  
    We observe a saturation effect where performance improves rapidly with initial increases in communication volume but stabilizes once objects are clearly resolved. The Sigmoid function captures this relationship of diminishing returns.
    The AP, $\rho_2$, is modeled using a Sigmoid-Exponential decay function, and is given by:
    \begin{equation}
    \label{eq:cor_sig_exp_decay}
    \rho_2(h_\text{s},b_\text{log})) = \left(\kappa \cdot \frac{1}{1+e^{-\lambda(b_\text{log}-\lambda_0)}} - \mu\right) \cdot e^{-\nu \cdot h_\text{s}},
    \end{equation}
    where $\kappa,\lambda,\lambda_0,\nu$ and $\mu$ are non-negative model parameters. The fitted curves across various AoI levels are illustrated in Fig.~\ref{fig:corridor_comm_vs_ap_aois}. Defining the penalty function and deriving the scheduling utility index $U(h,b)$ follows the identical methodology utilized in the intersection scenario.
    \begin{figure}[!t] 
      \centering
      \subfloat[]{
        \includegraphics[width=0.455\columnwidth]{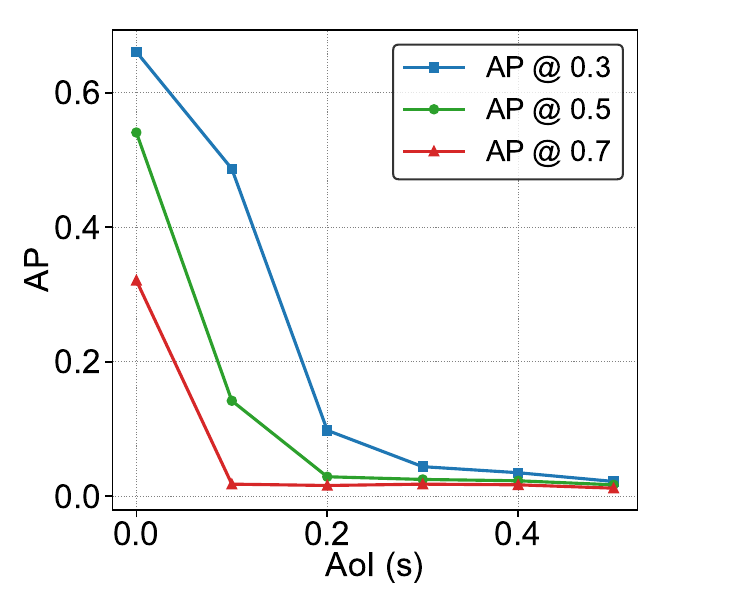} 
        \label{fig:corridor_aoi_vs_ap}
      }
    % \hfill
    \hspace{-0.5cm}
      \subfloat[]{
        \includegraphics[width=0.455\columnwidth]{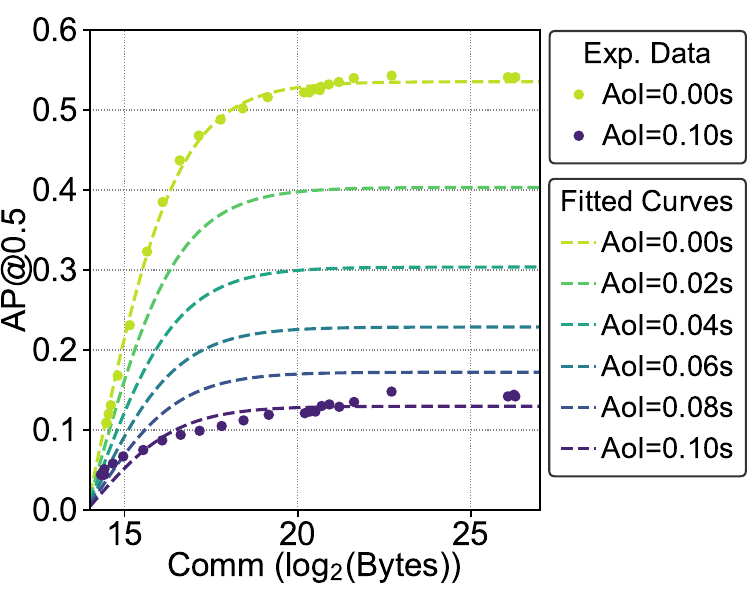} 
        \label{fig:corridor_comm_vs_ap_aois}
      }
      \caption{Penalty function fitting in the \textit{corridor} scenario. (a) shows the impact of AoI on AP and (b) depicts the impact of communication volume (in $\log_2(\text{Bytes})$) on AP@0.5 under various AoIs.}
      \label{fig:penalty_corridor}
    \end{figure}
    
    \begin{remark}
    Existing CP systems often incorporate mechanisms to compensate for latency, thereby enhancing performance in asynchronous environments \cite{wei2023asynchrony, wang2025latency, yu2025which2comm}. To ensure our empirically derived penalty function reflects this practical reality, we apply a linear performance compensation for the initial 100 ms of AoI, which limits the initial performance drop to just 0.02 in AP. We integrate a simplified compensation directly into our simulation procedure. 
    \end{remark}
    
    \section{Numerical Experiments}
    In this section, we evaluate the performance of our proposed TAMP scheduling algorithm for CP. The evaluation leverages the empirical penalty functions fitted on the RCooper dataset, as detailed in Section \ref{empirical}.  
    
    \subsection{Experimental Setup}
    We assess the performance of the scheduling algorithms by systematically evaluating their object detection accuracy measured by AP@0.5 across diverse environmental scenarios and varying key system parameters: the transmission rate, the communication budget, sensor-side computational power, and the BS scheduling capacity.
    To test the robustness of the TAMP algorithm, we utilize two types of environmental configurations. The homogeneous setup ensures all regions share identical characteristics, modeled in this experiment as a system consisting solely of corridor scenarios.
    Conversely, the heterogeneous setup comprises a mix of different scenario types, specifically an equal split between corridor and intersection scenarios in our experiment. This tests the  ability of algorithms to manage diverse penalty models simultaneously.

    The default simulation parameters are summarized in Table~\ref{tab:sim_params}. Reflecting a typical BS coverage area \cite{granda2021spatial}, we set the number of regions to A=20. The scheduling capacity is set to M=5 to account for limited parallel processing power. Following real-time network standards, the slot duration is $\tau=10$ ms \cite{zhang2023reliable}. The communication budget is set to $\Gamma_a=2$ Mb/slot , defined on masked features prior to a $32 \times$ transmission compression \cite{shu2025corange}. The channel rate follows $\mathcal{U}(1, 20)$ Mbps, consistent with IEEE 802.11p \cite{arena2019an}. Finally, processing delays are modeled by a shifted exponential distribution with shift $\psi=2$ and scale $\sigma=8$, yielding a 10 ms expected delay. These values apply unless specified otherwise.

    \begin{table}[htbp]
    \centering
    \caption{Default Simulation Parameters}
    \label{tab:sim_params}
    \begin{tabular}{l|l}
        \hline
        \textbf{Parameters} & \textbf{Values} \\
        \hline
        Number of Regions ($A$) & 20 \\
        Scheduling Capacity ($M$) & 5 \\
        Slot Length ($\tau$) & 10\,ms \\
        Communication Volume Budget ($\Gamma_a$) & 2\,Mb/slot \\
        Rate Distribution & $\mathcal{U}(1, 20)$\,Mbps \\
        Extraction Delay ($d^{\text{ext}}$) & Shifted Exp. ($\psi=2, \sigma=8$)\,ms \\
        Detection Delay ($d^{\text{det}}$) & Shifted Exp. ($\psi=2, \sigma=8$)\,ms \\
        Control Parameter ($V_a$) & $10^{-3}$ \\
        Simulation Horizon & 5000 slots \\ 
        \hline
    \end{tabular}
    \end{table}

    \subsection{Performance Comparison} \label{Performance Comparison}
    We compare TAMP algorithm against four baselines:
    \begin{itemize}
        \item \textbf{Age-Prio:} Schedules the top $M$ idle regions based on the highest current AoI, prioritizing the regions with the most stale information.
        \item \textbf{Rate-Prio:} Schedules the top $M$ idle regions based on the highest available transmission rates, prioritizing immediate communication efficiency.
        \item \textbf{GEA (Greedy Exchange Algorithm) \cite{qin2022aoi}:} It uses the expected AoI reduction for the current slot as its scheduling metric, defined as the difference between the projected AoI if the region remains idle versus if scheduled:
        \begin{equation} 
            \Pi_{a,k}^{\text{GEA}} = (h_{a,k} + \tau) - \mathbb{E}[d_{a,k}], 
        \end{equation} 
        with slot duration $\tau$ and the expected delay $\mathbb{E}[d_{a,k}]$.
        % It uses the expected AoI reduction for the current time slot as its scheduling metric, calculated as the difference between the AoI if the region remains idle versus if it is scheduled:
        % \begin{equation} \Pi_{a,k}^{\text{GEA}} = (h_{a,k} + \tau) - \mathbb{E}[d_{a,k}], \end{equation} where $\tau$ is the slot duration and $\mathbb{E}[d_{a,k}]$ is the expected service delay.
        \item \textbf{Max-Weight \cite{sun2023optimizing}:} A Lyapunov-based algorithm that defines scheduling priority using a metric that considers the long-term, cumulative impact of decisions:
        \begin{equation} \Pi_{a,k}^{\text{MW}} = \frac{W(h_{a,k})}{\bar{d}_a} - V Q_{a,k} b_{a,k}, \end{equation} 
        where $W(h_{a,k})$ is the weight function dependent only on AoI, derived in \cite{sun2023optimizing} that captures the scheduling worth.
    \end{itemize}
    
    \begin{figure*}[!t] 
    \centering
    \newcommand{\figwidth}{0.32\textwidth} 
    % Subfigure (a): Rate-AP (20 Corridors)
    \subfloat[]{
        \includegraphics[width=\figwidth]{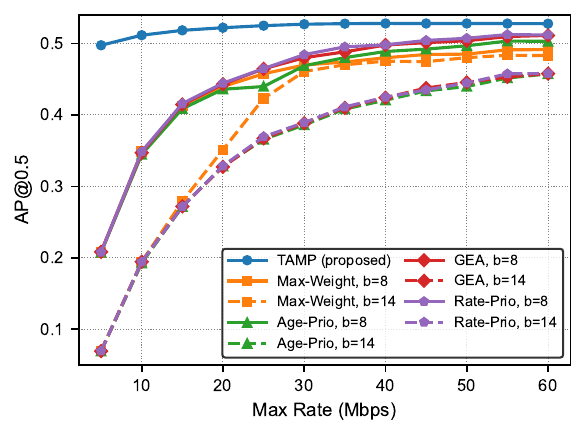}
        \label{fig:rate_20}
    }
    \hspace{-0.5cm}
    % Subfigure (c): Gamma-AP (20 Corridors)
    \subfloat[]{
        \includegraphics[width=\figwidth]{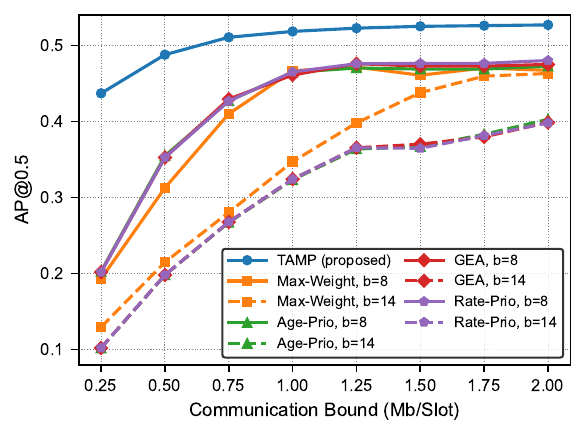}
        \label{fig:gamma_20}
    }
    \hspace{-0.5cm}
    % Subfigure (e): Delay-AP (20 Corridors)
    \subfloat[]{
        \includegraphics[width=\figwidth]{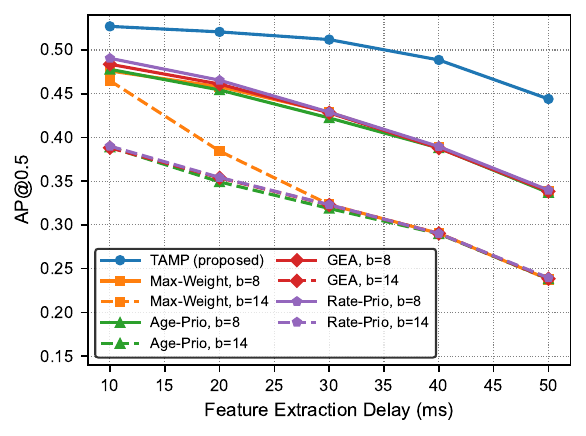}
        \label{fig:delay_20}
    }
    \caption{Performance comparison in \emph{homogeneous} (20 corridors) scenarios versus three key parameters: (a) transmission rate, (b) long-term communication bound, and (c) the expected of sensor-side feature extraction delay.
    }
    \end{figure*}
    
    \begin{figure*}[!t] 
    \centering
    \newcommand{\figwidth}{0.32\textwidth} 
    % Subfigure (a): Rate-AP (10 Corridors + 10 Intersections)
    \subfloat[]{
        \includegraphics[width=\figwidth]{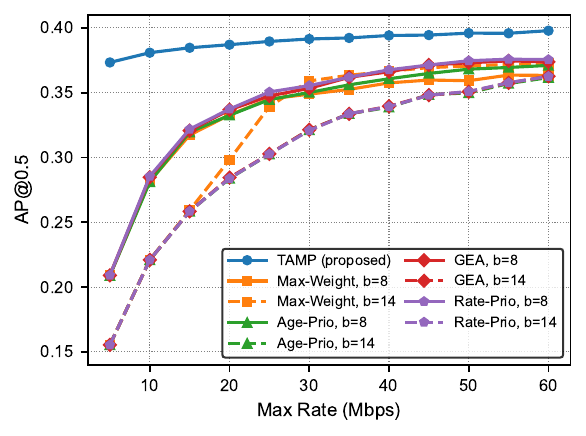}
        \label{fig:rate_10_10}
    }
    \hspace{-0.5cm}
    % Subfigure (b): Gamma-AP (10 Corridors + 10 Intersections)
    \subfloat[]{
        \includegraphics[width=\figwidth]{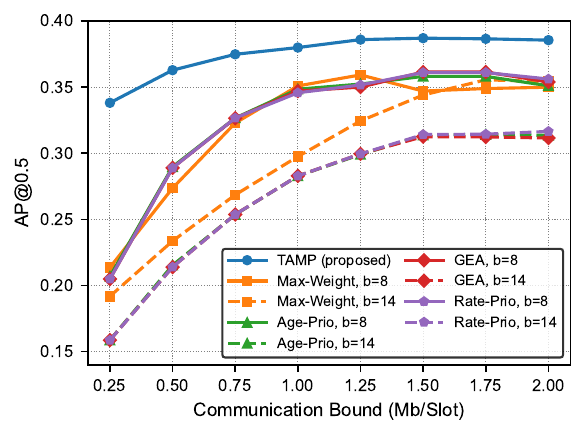}
        \label{fig:gamma_10_10}
    }
    \hspace{-0.5cm}
    % Subfigure (c): Delay-AP (10 Corridors + 10 Intersections)
    \subfloat[]{
        \includegraphics[width=\figwidth]{Figs/delay-ap_20cor_2bparam.pdf} 
        \label{fig:delay_10_10}
    }
    \caption{Performance comparison in \emph{heterogeneous} (10 corridors + 10 intersections) scenarios versus three key parameters: (a) transmission rate, (b) long-term communication bound, and (c) the expected of sensor-side feature extraction delay.    }
    \end{figure*}

    \subsubsection{Impact of the Transmission Rate}
    In this experiment, we assess the algorithm ability to adapt to different channel qualities. We vary the maximum available rate, $x$, of the uniform distribution $\mathcal{U}(1,x)$ from 5 to 60 Mbps. All other parameters are set to their default values as listed in Table~\ref{tab:sim_params}.
    The results for the homogeneous scenario are presented in Fig.~\ref{fig:rate_20}, while the results for the heterogeneous scenario are shown in Fig.~\ref{fig:rate_10_10}. In both scenarios, our proposed algorithm TAMP consistently and significantly outperforms all baselines. The performance gap is most pronounced in the low-to-mid rate regimes. For instance, at a maximum rate of 20 Mbps in the homogeneous scenario, TAMP improves the AP by over 15\% compared to the best-performing baseline. This is because the baseline algorithms employ a non-adaptive strategy that uses a fixed communication volume for each task, which was set to 8 Mb or 14 Mb in our experiments. In contrast, TAMP adaptively co-designs the scheduling decision and the communication volume for each region based on its instantaneous state, leading to more efficient resource utilization. While the performance of all algorithms saturates at very high rates as the system becomes limited by AoI, the adaptability of TAMP ensures it maintains the highest performance. 

    \subsubsection{Impact of the Communication Bound}
    Here, we investigate the impact of the long-term communication budget by varying $\Gamma_a$ from 0.25 to 2.0 Mb/Slot. Crucially, this budget is designed based on feature size after spatial-confidence masking but prior to the $32 \times$ compression applied immediately before transmission. All other parameters are set to their default values as specified in Table~\ref{tab:sim_params}. 
    The results illustrated in Fig.~\ref{fig:gamma_20} and Fig.~\ref{fig:gamma_10_10} show that TAMP consistently outperforms all baselines. While the performance of all algorithms improves with a larger budget before plateauing around $\Gamma=1.25$ Mb/Slot, the advantage of TAMP is most pronounced when resources are scarce. Specifically, under a tight budget of $\Gamma=0.75$ Mb/Slot, TAMP achieves a relative performance improvement of near 21\% over the best-performing baseline. This highlights the effectiveness of the Lyapunov-based algorithm of TAMP, which is designed to intelligently manage long-term constraints, in contrast to the myopic greedy baselines that struggle under stringent budget limitations.

    \subsubsection{Impact of the Sensor Computational Capability}
    To evaluate the algorithm robustness against varying sensor-side computational power, we adjust the mean of the stochastic feature extraction delay from 10 ms to 50 ms. Other parameters are kept at their default settings as shown in Table~\ref{tab:sim_params}.
    The results illustrated in Fig.~\ref{fig:delay_20} and Fig.~\ref{fig:delay_10_10}, show that while performance of all algorithms degrades with increased delay, TAMP algorithm consistently maintains the highest AP, and the superiority becomes even more pronounced as the processing delay grows. For instance, in the higher delay regime of 30 ms to 50 ms, TAMP achieves a relative performance improvement of 18\% to 29\% over the best-performing baseline. 
    \begin{figure*}[!t] 
    \centering
    \newcommand{\myfigwidth}{0.32\textwidth} 
        \begin{minipage}{\myfigwidth}
        \centering
        \includegraphics[width=\textwidth]{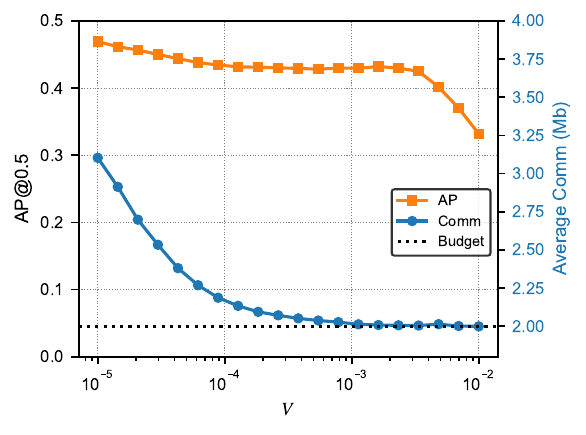} 
        \caption{Impact of trade-off parameter $V$.} 
        \label{fig:tradeoff_v}
    \end{minipage}
    \begin{minipage}{\myfigwidth}
        \centering
        \includegraphics[width=\textwidth]{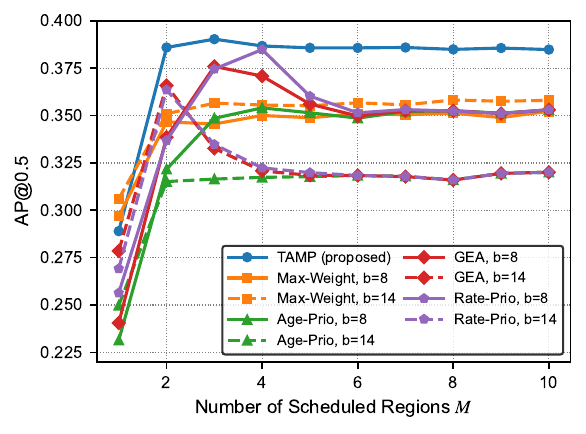} 
        \caption{Impact of BS Capacity $M$.} 
        \label{fig:multi_10_10_mix}
    \end{minipage}
    \begin{minipage}{\myfigwidth}
        \centering
        \includegraphics[width=\textwidth]{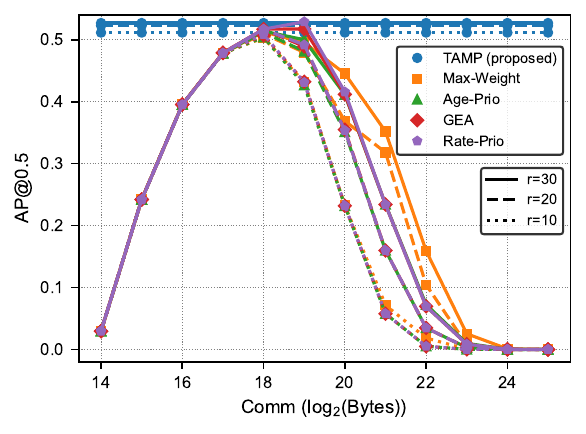}
        \caption{Impact of pre-set Comm $b$.} 
        \label{fig:adaptive_comm}
    \end{minipage}
    \end{figure*}
    \subsubsection{Impact of the BS Scheduling Capacity}
    This experiment tests the scalability of the algorithms with respect to the BS concurrent scheduling capacity, $M$. We vary $M$ from 1 to 10, while all other parameters follow the default configuration in Table~\ref{tab:sim_params}.
    The results depicted in Fig.~\ref{fig:multi_10_10_mix} show that as $M$ ranging from 2 to 4, the performance of TAMP forms the upper envelope of all baselines. In the plateau region where $M>5$, TAMP maintains a stable performance gain of approximately 9.5\% over the best-performing baseline. This demonstrates the ability of TAMP to effectively leverage additional scheduling capacity by intelligently selecting the most valuable regions.

    \subsection{Analysis of the Scheduling Trade-off}
    \subsubsection{Impact of the Trade-off Parameter $V$}
    We analyze the role of the trade-off parameter $V$. This parameter balance the long-term perception accuracy (AP) and the communication budget. 
    As shown in Fig.~\ref{fig:tradeoff_v}, a smaller $V$ value leads to a higher steady-state AP, as it places less emphasis on the immediate communication cost. Conversely, a larger $V$ value results in a much faster convergence to the communication budget.

    \subsubsection{Advantage of Adaptive Communication Volume Allocation}
    We demonstrate the significant advantage of TAMP to adaptively allocate the communication volume. In this experiment, the baselines are forced to use a fixed communication volume for every transmission, which we vary along the x-axis.
    The results is presented in Fig.~\ref{fig:adaptive_comm}. The rate distribution is set to a uniform distribution $\mathcal{U}(1,x)$ Mbps. The performance of the baselines are highly sensitive to the pre-set communication volume. Their performance curves first rise, as a larger volume allows for richer features, but then fall once the excessive volume leads to higher task delays. This shows that any fixed volume is only optimal under a narrow set of conditions. In contrast, our TAMP algorithm appears as a nearly horizontal line at the top of the plot, demonstrating a consistently high level of performance. This is because it adaptively calculates and deploys the optimal communication volume in each slot, rather than being constrained by a pre-set value. 

    \section{Conclusion}
    In this paper, we addressed the fundamental trade-off between perception accuracy and communication resource utilization in CP. We established that accurately modeling this trade-off and developing intelligent scheduling strategies are crucial for achieving efficient and reliable performance.
    Our primary contribution is the development of a systematic framework for co-designing communication and perception. We began by empirically analyzing a real-world dataset to characterize the non-linear relationships between AP, AoI, and communication volume in corridor and intersection scenarios. Based on this analysis, we derived a generalized penalty function to quantify performance degradation. Leveraging this function, we proposed the Timeliness-Aware Multi-region Prioritized (TAMP) scheduling algorithm, which adaptively allocates communication resources by simultaneously considering real-time channel conditions and information freshness.
    Extensive numerical experiments validated the superiority of our proposed method. Results demonstrate that TAMP consistently outperforms baselines including Age-Prio, Rate-Prio, GEA, and Max-Weight, achieving an AP improvement of up to 27\% across various configurations. Furthermore, our analysis highlights the capacity of TAMP for adaptively and efficiently allocating communication resources to balance detection accuracy against communication overhead.
    In summary, this work provides a theoretically principled framework for the joint design of communication and perception in CP systems. Our proposed penalty function and scheduling algorithm offer a practical solution for achieving high-performance, resource-efficient CP, paving the way for safer and more intelligent applications, from autonomous driving to large-scale smart city monitoring.
    
	\appendices{}
	\section{Proof of Lemma \ref{lem1}} 
    \label{sec:appendix_a}
    The long-term average penalty of $\mathcal{P}1$ is:
    \begin{equation}
        \limsup_{K\to\infty} \frac{1}{K} \mathbb{E}_\pi \left[ \sum_{k=0}^{K} f(h_k, \bm{b}_k) \right].
        \label{ltavg_pnt}
    \end{equation}
    We decompose the long-term average penalty into the sum of each interval. An interval is defined as the period between the completion time of two consequence tasks. For example, the $m$-th interval is from slot $k_m'$ to $k_{m+1}'$. Thus, \eqref{ltavg_pnt} yields:
    \begin{equation}
        \limsup_{K\to\infty} \frac{M_K}{K} \mathbb{E}_\pi \left[ \sum_{k=k'_m}^{k'_{m+1}} f(h_k, \bm{b}_k) \right].
        \label{sum_interval}
    \end{equation}

    % For the $m$-th scheduling, the CP task is initiated in slot $k_m$ and finishes in slot $k'_m$.
    The length of the $m$-th interval is $T_m$.
    Let $h_{k_m}$ be the AoI at the scheduling slot $k_m$. Let $\bm{b}_{k_m}$ be the communication volume allocated, which is constant during the $m$-th interval. Let $d_{k_m}$ be the total delay of that task started at slot $k_m$.
    % We define the total integral of penalty function during the $m$-th interval as $F_m$. 
    Based on the geometric pattern of the penalty function, as illustrated in Fig.~\ref{aoi_penalty}, the cumulative penalty over the $m$-th interval in \eqref{sum_interval} can be expressed as:
    \begin{align}
        \sum_{k=k'_m}^{k'_{m+1}} f(h_k, \bm{b}_k)  
        &=F(h_{k_{m+1}}+d_{k_{m+1}}+1, \bm{b}_{k_m}) \notag \\
        &\quad - F(d_{k_m}, \bm{b}_{k_m}), \label{cum_pen}
    \end{align}
    % where $F(h,\bm{b})$ is denoted as:
    % \begin{equation}
    %     F(h,\bm{b}) \triangleq \sum_{x=0}^h f(x,\bm{b}).
    % \end{equation}
    The expected length of that interval is:
    \begin{equation}
        T_m = h_{k_{m+1}}+d_{k_{m+1}}+1-d_{k_m}.
    \end{equation}
    When considering the long-term average performance, the expected interval duration under stationary conditions is
    \begin{equation}
        \mathbb{E}_\pi \left[ T \right] = \bar{h}+\bar{d}+1-\bar{d} = \bar{h}+1.
    \end{equation} 
    % Here, $h$ and $d$ are the random variables for the AoI before an update and the corresponding task delay, with averages $\bar{h} = \mathbb{E}[h]$ and $\bar{d} = \mathbb{E}[d]$, respectively.
    Due to the basic renewal theory, we yields:
    \begin{equation}
        \lim_{K\to\infty} \frac{M_K}{K} = \frac{1}{\mathbb{E}_\pi \left[ T \right]} = \frac{1}{\bar{h}+1}. \label{eq:renewal}
    \end{equation}

    To establish a tractable lower bound on the discrete cumulative penalty, we utilize the integral of $\Tilde{f}$ and the following inequality holds:
    \begin{equation}
        \sum_{x=0}^{h} f(x,\bm{b}) \ge \int_0^h \Tilde{f}(x,\bm{b})dx,
    \end{equation}
    The instantaneous penalty $\Tilde{f}(x,\bm{b})$ is a non-decreasing function of the AoI $x$ as more stale information incurs a greater penalty. Thus, $\Tilde{F}''(x,\bm{b})=\Tilde{f}'(x,\bm{b}) \ge 0$.
    Therefore, $\Tilde{F}(h,\bm{b})$ is a convex function of $h$.
    
    Note that $\bm{b}_{k_m}$ is independent to $h_{k_{m+1}}$ and $d_{k_{m+1}}$.
    Taking the long-term average expectation and applying the Jensen's inequality to \eqref{cum_pen}, which gives:
    \begin{align}
        &\mathbb{E}_\pi[F(h_{k_{m+1}}+d_{k_{m+1}+1}, \bm{b}_{k_m})]-\mathbb{E}_\pi[F(d_{k_m}, \bm{b}_{k_m})] \nonumber \\
        &\ge \Tilde{F}(\bar{h} + \bar{d}+1, \bm{b}_{k_m})-\mathbb{E}_\pi[F(d_{k_m}, \bm{b}_{k_m})]. 
        \label{eq:jensen's}
    \end{align}
        
    % Taking the expected value of the second term of \eqref{cum_pen} yields $\mathbb{E}[F(d, \bm{b})]$.
    Under stationary conditions, the random variables $h$, $d$, and $\bm{b}$ are identically distributed across different intervals $m$, hence we drop the subscript $k_m$.
    Substituting \eqref{eq:renewal} and \eqref{eq:jensen's} into \eqref{sum_interval}  yields the final lower bound for the long-term average penalty:
    \begin{equation}
        \bar{f} \ge \frac{\Tilde{F}(\bar{h} + \bar{d}+1, \bm{b}) - \mathbb{E}_\pi[F(d, \bm{b})]}{\bar{h}+1}.
    \end{equation}
    \qed

    \section{Proof of Lemma \ref{lem:convex}}
    \label{sec:appendix_b}
    The objective function of Problem $\mathcal{P}2$, where $h$ is the optimization variable, is: \begin{equation} 
        \frac{1}{h} \left( \Tilde{F}(h+\bar{d}, \bm{b}) - \mathbb{E}[F(d,\bm{b})] \right). 
    \end{equation} For fixed parameters $\bm{b}$ and $\bar{d}$, and $\mathbb{E}[F(d,\bm{b})]$ is treated as a constant term. We will prove the quasi-convexity of: 
    \begin{equation} 
        y(h) = \frac{1}{h} \left( \Tilde{F}(h+\bar{d}) - C \right),\quad h>0,
    \end{equation} 
    where $C$ is a constant. 
    
    Based on the physical meaning of the instantaneous penalty, the value loss of information increases as time evolves, therefore $\Tilde{f}(t)$ is a non-decreasing function of $t$. 
    This implies that the cumulative penalty $\Tilde{F}(h) = \int_0^h \Tilde{f}(t) dt$ satisfies:
    \begin{equation}
        \Tilde{F}''(h) = \Tilde{f}'(h)\ge 0.
    \end{equation}
    Thus, $\Tilde{F}(h)$ is a convex function. Defined by composition with an affine function, $\Tilde{F}(h+\bar{d})$ is also a convex function. 
    
    A function $y(h)$ is quasi-convex if and only if its sublevel set $S_\gamma=\{h|y(h)\le \gamma\}$ is a convex set for every $\gamma \in \mathbb{R}$. For a given $\gamma$, the sublevel set $S_\gamma$ is defined by the inequality: 
    \begin{equation}
        \frac{\Tilde{F}(h+\bar{d}) - C}{h} \le \gamma.
    \end{equation}
    Since $h>0$, we have:
    \begin{equation}
        \Tilde{F}(h+\bar{d}) - \gamma h \le C.
    \end{equation}
    
    We define the function $H(h) = \Tilde{F}(h+\bar{d}) - \gamma h$. The first term $\Tilde{F}(h+\bar{d})$ has been proved as a convex function. The second term $-\gamma h$ is a linear function, which is also convex.
    
    Since the sum of two convex functions is convex, $H(h)$ is a convex function. The sublevel set $S_\gamma=\{h|H(h)\le C\}$, which is the sublevel set of a convex function. The sublevel set of any convex function is always a convex set. 
    Thus, $y(h)$ is a quasi-convex function, $\mathcal{P}2$ is a quasi-convex optimization problem.
    \qed
    
   \section{Proof of Theorem \ref{theorem:dpp}}
    \label{sec:appendix_c}
    % We begin by establishing an upper bound on the one-slot conditional drift, $\Delta(Q_k)$. 
    For simplification, we omit the subscript $a$.
    Let $L(Q_k) \triangleq \frac{1}{2} Q_k^2$ be the quadratic Lyapunov function for the virtual queue $Q_k$ defined in \eqref{eq:virtual_queue}. The drift can be bounded by first analyzing the change in the squared queue length:
    \begin{align}
        Q_{k+1}^2 - Q_k^2
        &= \left( \max\{ Q_k + b_k - \Gamma, 0 \} \right)^2 - Q_k^2 \notag \\
        &\le (b_k - \Gamma)^2 + 2 Q_k (b_k - \Gamma),
    \end{align}
    where the inequality follows from the property $(\max\{x,0\})^2 \le x^2$. 
    
    Taking the conditional expectation of $\frac{1}{2}(Q_{k+1}^2 - Q_k^2)$ given $Q_k$ yields the drift bound:
    \begin{equation}
        \Delta(Q_k) \le \frac{1}{2} \mathbb{E} \left[ (b_k - \Gamma)^2 \mid Q_k \right] + Q_k \, \mathbb{E} \left[ b_k - \Gamma \mid Q_k \right].
    \end{equation}
    Substituting the control decision $b_k = u_k b_k^*$ and rearranging the terms gives:
    \begin{align}
        \Delta(Q_k) \le \underbrace{\frac12 \mathbb{E} \!\left[ (u_k b_k^* - \Gamma)^2 \,\middle|\, Q_k \right]}_{\text{bounded by a constant } C}
        + Q_k \mathbb{E} \!\left[ u_k b_k^* \mid Q_k \right] - Q_k \Gamma.  \nonumber
        \label{eq:final_drift_bound_in_proof}
    \end{align}
    Now, we add the penalty term to both sides to obtain an upper bound on the drift-plus-penalty expression:
    \begin{align}
        & \Delta(Q_k) - V \, \mathbb{E}[U(h_k, b_k^*) \, u_k \mid Q_k] \notag \\
        & \le C + Q_k \mathbb{E} \!\left[ u_k b_k^* \mid Q_k \right] - Q_k \Gamma - V \, \mathbb{E}[U(h_k, b_k^*) \, u_k \mid Q_k].
        \nonumber
    \end{align}
    To minimize this upper bound in each slot, the algorithm must choose $u_k$ to minimize the right-hand side of the inequality. The terms $C$ and $-Q_k \Gamma$ are constant with respect to the optimization variable $u_k$ and can be dropped, yielding the desired result in \eqref{eq:per_slot_objective}.
    \qed

\bibliographystyle{IEEEtran}
\bibliography{journal_refs}

\end{document}